\documentclass[final,12pt]{colt2022} % Include author names
\usepackage{latexsym}
\usepackage{color}
\usepackage{graphics}
\usepackage{cancel}
\usepackage{mathtools}
\usepackage[mathscr]{euscript}
\usepackage{hyperref}
\usepackage{enumitem}

\let\vec\undefined
\usepackage{MnSymbol}
\DeclareMathAlphabet\mathbb{U}{msb}{m}{n}
\usepackage{xpatch}

\let\Pr\undefined

\DeclareMathOperator*{\Pr}{\mathbb{P}}

\DeclareMathOperator*{\E}{\mathbb E}
\DeclareMathOperator*{\poly}{poly}
\DeclareMathOperator*{\conv}{conv}
\DeclareMathOperator*{\argmax}{argmax}

\DeclareMathOperator{\nbr}{nbr}

\DeclareMathOperator{\Ind}{\mathbb{I}}

\newtheorem{question}{Question}

\newcommand{\eps}{\varepsilon}

\newcommand{\Reg}{\mathsf{Reg}}

\newcommand{\Val}{\mathsf{Val}}
\newcommand{\ConVal}{\mathsf{PerConVal}}
\newcommand{\CorrVal}{\mathsf{CorrVal}}
\newcommand{\SwapReg}{\mathsf{SwapReg}}
\newcommand{\LinSwapReg}{\mathsf{LinSwapReg}}
\newcommand{\PolySwapReg}{\mathsf{PolySwapReg}}
\newcommand{\BR}{\mathsf{BR}}

\newcommand{\bself}{\overline{0}}

\newcommand{\cA}{\mathcal{A}}

\newcommand{\cD}{\mathcal{D}}
\newcommand{\cF}{\mathcal{F}}

\newcommand{\cM}{\mathcal{M}}

\newcommand{\cP}{\mathcal{P}}

\newcommand{\cQ}{\mathcal{Q}}

\newcommand{\cV}{\mathcal{V}}

\newcommand{\ignore}[1]{}

\makeatletter
\newtheorem*{rep@theorem}{\rep@title}
\newcommand{\newreptheorem}[2]{%
\newenvironment{rep#1}[1]{%
 \def\rep@title{#2 \ref{##1}}%
 \begin{rep@theorem}}%
 {\end{rep@theorem}}}
\makeatother

\hypersetup{
  colorlinks   = true,
  urlcolor     = blue,
  linkcolor    = blue,
  citecolor    = blue
}

\newreptheorem{theorem}{Theorem}
\newreptheorem{lemma}{Lemma}
\newreptheorem{corollary}{Corollary}
\newreptheorem{proposition}{Proposition}

\title[Strategizing against Learners in Bayesian Games]{Strategizing against Learners in Bayesian Games}
\usepackage{times}
\coltauthor{%
 \Name{Yishay Mansour} \Email{mansour@google.com}\\
 \Name{Mehryar Mohri} \Email{mohri@google.com}\\
 \Name{Jon Schneider} \Email{jschnei@google.com}\\
 \Name{Balasubramanian Sivan} \Email{balusivan@google.com}\\
 \addr Google Research%
}

\begin{document}

\maketitle

\begin{abstract}%
  We study repeated two-player games where one of the players, the learner, employs a no-regret learning strategy, while the other, the optimizer, is a rational utility maximizer. We consider general Bayesian games, where the payoffs of both the optimizer and the learner could depend on the type, which is drawn from a publicly known distribution, but revealed privately to the learner. We address the following questions: (a) what is the bare minimum that the optimizer can guarantee to obtain regardless of the no-regret learning algorithm employed by the learner? (b) are there learning algorithms that cap the optimizer payoff at this minimum? (c) can these algorithms be implemented efficiently? While building this theory of optimizer-learner interactions, we define a new combinatorial notion of regret called polytope swap regret, that could be of independent interest in other settings.
\end{abstract}


\section{Introduction}
How should one play a two-player repeated game? A commonly employed strategy when dealing with a repeated setting is to use a no-regret learning algorithm. Such algorithms assign higher weight to actions that achieved good performance in previous rounds of the game. An important danger lurks when one uses a learning algorithm to play a repeated game: the opponent (who we will call the ``optimizer'') could be a rational utility-maximizer who might try to explicitly exploit the fact that the learning algorithm chooses its actions based on past performance. Can one design learning algorithms that do not get fooled because they learn from past actions? Can we precisely characterize the class of learning algorithms that are robust from being manipulated in this way? What are the meaningful outcomes and benchmarks when studying this optimizer-learner interaction?

Recent work by~\cite{deng2019strategizing} initiated the study of optimizer-learner interactions in general 2-player bimatrix games. They showed that regardless of, and without knowledge of, the specific no-regret-learning algorithm used by the learner, the optimizer can always guarantee himself at least the Stackelberg value\footnote{The Stackelberg variant of a two-player game is a one-shot two-stage game where the optimizer moves first and publicly commits to a (possibly mixed) strategy, and the learner then best responds to this strategy. The equilibrium that results in this two-stage game when both players play optimally is called a Stackelberg equilibrium. We note here that the optimizer's utility in a Stackelberg equilibrium is at least as high as the utility he can get in any (pure or mixed-strategy) equilibrium thereby showing that playing against a learner is more beneficial than playing against another optimizer (i.e., a rational utility maximizer). Formally defined in Section~\ref{sec:prelim}.} of the game by playing a static fixed strategy each round. 
%Note that the Stackelberg value is guaranteed to weakly exceed the Nash equilibrium payoff for the optimizer, thereby showing that playing against a learner is beneficial than playing a rational utility maximizer. 
More interestingly, they show that for a large class of no-regret learning algorithms called mean-based algorithms, there are games where the optimizer can badly mislead the learner and profit immensely by playing a \textit{dynamic strategy} that varies over time. In particular, the optimizer can architect situations where the learners' responses in certain rounds are far from their best response, owing to the force of memory that is inherent in these learning algorithms.  However, \cite{deng2019strategizing} also show that if the learner were to use a more sophisticated learning algorithm, namely, a no-swap-regret algorithm, then the optimizer is unable to get anything more than the utility he is able to get in the Stackelberg equilibrium of the game. I.e., a no-swap-regret algorithm is \emph{sufficient} to prevent the optimizer from benefiting from dynamic strategic behavior.

\paragraph{Questions.} In this paper, we primarily focus on two questions. First, the results of~\cite{deng2019strategizing} immediately motivate the following question: in general 2-player games (the same set of games studied in~\cite{deng2019strategizing}; we call these \emph{standard} games), is a no-swap-regret algorithm also \emph{necessary} for the learner to cap the optimizer's payoff of Stackelberg value? Or can the learner run algorithms that, despite having large swap regret, cap the optimizer's payoff at the Stackelberg value of the game? In other words, we seek to characterize the precise class of learning algorithms that ensure that the optimizer cannot benefit from dynamic strategic behavior. 

Second, we seek to develop the theory of optimizer-learner interaction in the significantly more general class of Bayesian games, and develop a complete understanding of the landscape there. These games arise naturally in various economic settings, for example an optimizer selling an item to a learner where the learner's private value for the item is their type \citep{braverman2018selling}. Formally, a Bayesian game begins with one of $C$ \textit{contexts} (``types'') $c \in [C]$ being drawn from a publicly known distribution $\cD$ with probability $p_c$ of outputting context $c$. This context $c$ is told to the learner but not to the optimizer. Based on the context $c$, the learner chooses an action $j \in [N]$; simultaneously, the optimizer chooses an action $i \in [N]$. The optimizer then receives utility $u_{O}(i, j, c)$, and the learner receives utility $u_{L}(i, j, c)$ -- note that we allow both utilities to depend on the context $c$. I.e., instead of a single bi-matrix in the \emph{standard} game, a Bayesian game can be thought of as specified by $C$ bi-matrices. 

As in standard games, it is straightforward to show that by playing a static fixed strategy, an optimizer can obtain at least the \emph{Bayesian Stackelberg} value of the game per round, as long as the learner runs a no-(contextual)-regret learning algorithm (we prove this in Lemma~\ref{lem:stack_achieve}). For the Bayesian setting, we would like to understand (a) are there learning algorithms that are robust to dynamic strategic behavior (that cap the optimizer payoff at this Stackelberg value)? in particular, what is the right generalization of swap regret? (b) can these algorithms be implemented efficiently? if not, what guarantees can we provide for efficient learning algorithms?

\subsection{Our Results} 

For the first question on standard games, we show that no-swap-regret algorithms are the precise class of algorithms that are robust against dynamic strategic behavior of the optimizer. Specifically, for any learning algorithm that has a swap-regret of $R$, we construct games where the optimizer \emph{can earn $R/2$ more than the Stackelberg value of the game}. In particular therefore, if the learner had, say, a linear swap regret, the optimizer would earn linearly more than the Stackelberg value of the game. The main novelty in the proof of this result lies in the game construction: the payoffs in the game we design should ensure a delicate balance between the optimizer's payoff being high, while the Stackelberg value not being too high. We present the precise details of the construction in Section~\ref{sec:standard}. Apart from completing the picture for standard games, this result also provides a new characterization of swap regret as a measure of robustness against strategic behavior. 

For Bayesian games, the challenges are multi-fold. Unlike standard games, it is not clear what the correct generalization of swap regret should be. Many seemingly natural choices of regret definition turn out to be incorrect. For example, motivated by the fact that running an independent low external regret algorithm per context results in low external regret, one may consider running an independent low swap regret algorithm for each context. But this does not work! In particular, there are simple games where an optimizer can earn linearly more than the Bayesian Stackelberg value by playing a dynamic strategy against such a learning algorithm (we give one example in Theorem~\ref{thm:alg1bad}).  

In this paper we provide a nuanced generalization of swap-regret to the Bayesian setting that we call \emph{polytope swap regret}. We prove that this notion of regret has the guarantee that any learner playing a learning algorithm with $o(T)$ polytope regret \emph{is guaranteed to asymptotically cap the optimizer payoff at the Bayesian Stackelberg value}. While we do not yet know that a low polytope-swap regret is necessary to cap the optimizer payoff at the Bayesian Stackelberg value (i.e., that it is tight in the same way swap regret is for standard games), we provide another generalization of swap regret called \emph{linear swap regret} such that \emph{a low linear swap regret is necessary to cap the optimizer payoff at the Bayesian Stackelberg value}. 

Polytope swap regret actually extends far beyond just the Bayesian setting, and can be thought of as a generalization of swap regret to the setting of online linear optimization. The idea behind polytope swap regret stems from viewing the learner's actions -- mappings from contexts to a distribution over actions -- as points in the polytope $\cP = \Delta([N])^{C} \subseteq \mathbb{R}^{N \times C}$. Our ``swap functions'' then allow the vertices of this polytope to be swapped with each other. Every point inside the polytope (including the learner's actions) can be written as a convex combination of the vertices of $\cP$, to which this swap function can be applied. Of course, there may be many ways to write a given point as a convex combination of vertices: we consider the most permissive definition of regret by choosing the decomposition that leads to the least regret in hindsight. I.e., we say that polytope swap regret is high only if every vertex decomposition of the learner's actions will generate high regret by applying a swap function to the vertices. Precise definitions are given in Section~\ref{sec:bayesian} and~\ref{sec:psr}. When we restrict these swap functions to be linear maps, we obtain the linear swap regret. Interestingly, when $\cP$ is the simplex $\Delta([N])$ both these notions of swap regret (polytope and linear), are equal (it is possible to implement any swap function on vertices via a linear map), and both reduce to the ordinary notion of swap regret.

These twin concepts, with polytope swap regret being sufficient and linear swap regret being necessary, raise the question of which of these could potentially be both necessary and sufficient. To answer this, we first show that these two concepts are not the same by separating polytope and linear swap regret (Theorem~\ref{thm:poly_lin_sep}). We then show in Theorem~\ref{thm:poly_lin_sep_bayesian} that linear swap regret is not a sufficient condition. We conjecture that polytope swap regret is the right notion that captures robustness against strategic behavior in Bayesian games, by being both necessary and sufficient. We leave the necessity of polytope swap regret as a concrete open question.

\paragraph{Efficient Algorithms for Bayesian Games.} A natural question is whether we can given efficient algorithms for these generalizations of swap regret. We address this issue in Appendix~\ref{sec:algcons}. In particular, we show how to construct a low polytope swap regret algorithm given any low swap regret algorithm as input, such that this algorithm incurs a regret of at most $O(\sqrt{TV\log V})$ and runs in time $O(\poly(V))$, where $V$ is the number of vertices of the polytope under consideration. For Bayesian games, the number of vertices of the polytope $\cP = \Delta([N])^{C}$ is $N^C$, and thus, in cases where the number of contexts $C$ is small, these bounds are manageable. When $C$ is large, we analyze two other algorithms: the ``low-swap-regret per context'' algorithm mentioned above, and a generalization of the external to internal regret reduction of \cite{BlumMansour2007}. While neither algorithm (provably) has low polytope regret, we show they both provide some robustness by capping the optimizer's payoff at some variant of the Bayesian Stackelberg value of the game. Finally, we conjecture this exponential dependence on $C$ is necessary to achieve low polytope regret -- we contribute one piece of evidence towards this by showing that the Bayesian Stackelberg value itself is APX-hard to compute in general Bayesian games (Theorem \ref{thm:apx-hardness}).

\subsection{Related Work}
While the introduction discusses the work closest to ours, namely~\cite{deng2019strategizing}, we discuss further related work in detail in Appendix~\ref{sec:app_related}. 

\section{Model and preliminaries}
\label{sec:prelim}

\paragraph{Notation} We write $[N]$ to denote the set $\{1, 2, \dots, N\}$. For a finite set $S$, we write $\Delta(S)$ to denote the set of distributions over $S$. We defer most proofs to Appendix \ref{app:omitted} for the sake of brevity.

\subsection{Games and equilibria}\label{sec:intro_games}

We begin this paper by considering finite bimatrix games (which we refer to as \textit{standard games}). A standard game is a game between two players, who we refer to as the \textit{optimizer} and the \textit{learner}. The optimizer must choose one of $M$ actions (labeled $1$ through $M$) and learner must simultaneously choose one of $N$ actions (labeled $1$ through $N$). If the optimizer chooses action $i \in [M]$ and the learner chooses action $j \in [N]$, then the optimizer receives utility $u_{O}(i, j)$ and the learner receives utility $u_{L}(i, j)$. We will assume all utilities are bounded in $[-1, 1]$ (so $|u_{O}(i, j)| \leq 1$, and $|u_{L}(i, j)| \leq 1$). To simplify analysis in the sections that follow, we will eliminate the role of randomness (which is mostly tangential to the main points of this paper) by allowing the optimizer and learner to directly play mixed strategies in $\Delta([M])$ and $\Delta([N])$, and deterministically receive the corresponding expected reward. That is, when the optimizer plays $\alpha \in \Delta([M])$ and $\beta \in \Delta([N])$, the optimizer's utility is given (deterministically) by $u_{O}(\alpha, \beta) = \sum_{i=1}^{M}\sum_{i=1}^{N} \alpha_{i}\beta_{j}u_{O}(i, j)$ (and the learner's utility is computed similarly).

In the second part of this paper, we extend our study to a specific subclass of Bayesian games where the learner is randomly assigned a \emph{type}, unknown to the optimizer (we refer to such games simply as \textit{Bayesian games}). Such games arise naturally in various economic settings, for example an optimizer selling an item to a learner where the learner's private value for the item is their type \citep{braverman2018selling}. More formally, a Bayesian game begins with one of $C$ \textit{contexts} (``types'') $c \in [C]$ being drawn from a publicly known distribution $\cD$ with probability $p_c$ of outputting context $c$. This context $c$ is told to the learner but not to the optimizer. Based on the context $c$, the learner chooses an action $j \in [N]$; simultaneously, the optimizer chooses an action $i \in [N]$. The optimizer then receives utility $u_{O}(i, j, c)$, and the learner receives utility $u_{L}(i, j, c)$ -- note that we allow both utilities to depend on the context $c$. 

As with standard games, we eliminate the role of randomness in Bayesian games by allowing the optimizer and learner to play mixed strategies and assigning rewards deterministically. As with standard games, the optimizer plays a mixed strategy $\alpha \in \Delta([M])$. The learner simultaneously plays a function $\beta \colon [C] \to \Delta([N])$ mapping contexts to mixed strategies (representing which strategy the learner would play for each context). The optimizer then receives reward 

$$u_{O}(\alpha, \beta) = \sum_{k=1}^{C}p_{k}u_{O}(\alpha, \beta(c_k), c_k) = \sum_{c=1}^{C}\sum_{i=1}^{M}\sum_{i=1}^{N} p_{c}\alpha_{i}\beta(c)_{j}u_{O}(i, j, c).$$

\noindent
The learner's reward is computed similarly.

We are interested in settings where the optimizer and learner repeatedly play a game for $T$ rounds. We write $\alpha^t$ to denote the optimizer's strategy in round $t$ and $\beta^t$ to denote the learner's strategy in round $t$. We will also insist that this repeated game is \textit{full information}, in the sense that after each round, either player should be able to figure out their counterfactual utility if they had played a different mixed strategy that round (for example, this is the case when the mixed actions $\alpha^t$ and $\beta^t$ of both players are made publicly known after round $t$ and that both the optimizer and learner have full knowledge of their utility functions). This will allow the learner to play this game by running the learning algorithms detailed in the next section.

\subsection{Learning algorithms, regret, and swap regret}\label{sec:intro_learning}

% \jon{todo: fix how notation switches between $\cB$ and $[N]$}

As their name suggests, the learner will play the game by running an online learning algorithm to select their actions. We will consider the following (fractional, full-information, and deterministic) model for online learning:

A learner will face a decision between $N$ actions for each of $T$ rounds. An adversary begins by obliviously\footnote{Since the learner here is deterministic, it actually does not make a difference whether we allow the adversary to be adaptive or not; an oblivious adversary can simply simulate the actions of the learner.} selecting $T$ reward vectors $r^1, r^2, \dots, r^{T} \in [0, 1]^N$, where $r^{t}_{i}$ represents the reward if the learner picks action $i$ in round $t$. Then, for each round $1 \leq t \leq T$ the learner selects a distributional action $\beta^{t} \in \Delta([N])$ deterministically as a function of the reward vectors in previous rounds, i.e., $r_1, r_2, \dots, r_{t-1}$. The learner then receives reward $\sum_{j=1}^{N} \beta^{t}_j r^{t}_j$ and the full reward vector $r_{t}$ for round $t$ is revealed to the learner.

Note that a learner can use such a learning algorithm to play in standard games. We evaluate a learning algorithm by providing bounds on some form of ``regret''. We consider two such notions: the \textit{external regret} of an algorithm, and the \textit{swap regret} of an algorithm. The external regret (or simply ``regret'') of a learning algorithm on a specific problem instance represents the gap between the total reward obtained by the learning algorithm and the best reward obtainable by playing the best single fixed action in hindsight; it is given by:
\begin{align*}
\Reg = \left(\max_{j^{*} \in [N]} \sum_{t=1}^{T} r^{t}_{j^*}\right) - \left(\sum_{t=1}^{T}\sum_{j=1}^{N} \beta^{t}_{j} r^{t}_{j} \right).
\end{align*}

A learning algorithm is \textit{low-regret} if it sustains $o(T)$ external regret on any problem instance with $T$ rounds (and a fixed number of actions). It is well known that there exist efficient low-regret algorithms in this setting which sustain regret at most $O(\sqrt{T\log N})$ \citep{LittlestoneWarmuth1994,FreundSchapire1997}. Interestingly, the property of being low-regret is not sufficient to guarantee good performance in the optimizer-learner settings described in Section \ref{sec:intro_games}; there are games where the optimizer can get much more than their Stackelberg value if the learner plays certain low-regret algorithms. As we will show, to guarantee that this does not occur, the learner must play an algorithm with low swap-regret. %\citep{BlumMansour2007}

% \footnote{YM: We need to explain why is the Stackelberg value an interesting benchmark.}

The swap regret of a learning algorithm on a specific problem instance represents the gap between the total reward obtained by the learning algorithm, and the maximum award they could obtain in hindsight if they had applied a deterministic \textit{swap function} to their actions (i.e., playing action 2 instead of action 1 every time they played action 1 with any weight). Formally, we can define the swap regret as follows:

\begin{align*}
\SwapReg = \left(\max_{\pi: [N]\rightarrow[N]} \sum_{t=1}^{T}\sum_{j=1}^{N} \beta^{t}_{j} r^{t}_{\pi(j)}\right) - \left(\sum_{t=1}^{T}\sum_{j=1}^{N} \beta^{t}_{j} r^{t}_{j} \right).
\end{align*}

A learning algorithm is \textit{low-swap-regret} if it sustains $o(T)$ swap regret on any problem instance with $T$ rounds. As with external regret, it is known there exist low-swap-regret algorithms. For example, the construction of \citet{BlumMansour2007} demonstrates how to devise an algorithm with swap regret $O(\sqrt{TN\log N})$.

Until now, we have described a form of online learning that can be used to play standard games. To play Bayesian games, we need a form of online \textit{contextual} learning -- we defer discussion of this to the beginning of Section \ref{sec:bayesian}.

\subsection{Stackelberg equilibria and strategies}\label{sec:intro_equilibria}
\label{sec:StackVal}
One of the primary benchmarks that we will use to measure the performance of the optimizer is the optimizer's value in the Stackelberg equilibrium of the one-shot game.

Let $G$ be a standard game, and for each mixed strategy $\alpha \in \Delta([M])$, define the learner's best-response function $\BR(\alpha) = \argmax_{j \in [N]} u_{L}(\alpha, j)$. We then define the Stackelberg value of $G$ to be the value

$$\Val(G) = \max_{\alpha} \max_{\beta \in \BR(\alpha)} u_{O}(\alpha, \beta).$$

We can similarly define the Stackelberg value $\Val(G)$ for a Bayesian game, with the only difference that now the learner's best response $\BR(\alpha) = \argmax_{\beta \in [N]^{[C]}} u_{L}(\alpha, j)$ is taken over all strategies $\beta(c) : [C] \rightarrow [N]$ mapping contexts to actions. 

Intuitively, the Stackelberg value represents the maximum value the optimizer can obtain by playing a fixed strategy against a strategic learner. Note that: a) we allow the optimizer to play a mixed strategy instead of just a pure strategy (so this is what is occasionally referred to as a Stackelberg mixed strategy, e.g. in \cite{conitzer2016stackelberg}) and b) we break ties for the learner in favor of the optimizer. 

The Stackelberg value is a benchmark that arises naturally in our setting for the following reason\footnote{For the case of standard games, this was shown in \cite{deng2019strategizing}; we include the straightforward generalization to Bayesian games (and more generally, polytope games) in Appendix \ref{app:achieve_stack}.}: if the optimizer is playing a game $G$ for $T$ rounds against a learner running a low-regret algorithm, then the optimizer can guarantee (under some mild conditions on $G$) that they receive reward at least $\Val(G)T - o(T)$. Moreover, the optimizer can accomplish this by playing their fixed Stackelberg strategy every round. The central goal of this paper is to understand when the optimizer can significantly outperform this benchmark by playing a dynamic strategy.

\section{Standard games}\label{sec:standard}

We begin our discussion with standard games. In \cite{deng2019strategizing}, the authors show that if an optimizer is playing a learner with low swap regret, the optimizer can get no more than $\Val(G)T + o(T)$ utility. We complement this result by showing that low swap regret is \textit{necessary}; if a learner is playing an algorithm that is not low swap regret, then it is possible to construct a game $G$ where an optimizer can gain significantly ($\Omega(T)$) more than their Stackelberg value by playing some dynamic strategy against this learner. 

The main argument is encapsulated in the following lemma, which shows how to convert a high swap-regret online learning instance for the learner into a game where the optimizer outperforms Stackelberg.

\begin{lemma}\label{lem:standard}
Let $\cA$ be a learning algorithm which incurs swap-regret $R$ on some online learning instance with $N$ actions and $T$ rounds. Then there exists a standard game $G$ (with $N$ actions for the learner and $M = 2^N$ actions for the optimizer) such that if an optimizer plays $T$ rounds of $G$ against a learner running $\cA$, the optimizer can receive a total reward of $\Val(G)T + \frac{1}{2}R$. 
\end{lemma}
\begin{proof}
Recall that in our model, an online learning instance with $N$ actions and $T$ rounds is completely specified by a sequence of $T$ reward vectors $r^{t} \in [-1, 1]^N$. Fix $r^{t}$ to be the bad instance mentioned in the theorem statement for algorithm $\cA$, and let $\beta^t \in \Delta([N])$ denote the action of the learner in round $t$. Since the regret of $\cA$ on this bad instance is $R$, this implies that for some swap function $\pi:[N] \rightarrow [N]$, we have that

\begin{equation}\label{eq:hiswap}
    \left(\sum_{t=1}^{T}\sum_{j=1}^{N} \beta^{t}_{j} r^{t}_{\pi(j)}\right) - \left(\sum_{t=1}^{T}\sum_{j=1}^{N} \beta^{t}_{j} r^{t}_{j} \right) = R. 
\end{equation}

We will now show how to use the reward $r^{t}$ to construct a game where the optimizer can do better than the Stackelberg equilibrium. For now, we will construct a game $G$ where the learner has $N$ actions and the optimizer has $M=T$ actions -- we will later show how to decrease $M$ to a value independent of $T$.

We begin by specifying the learner's payoffs. Unsurprisingly, these will be drawn directly from the learner's rewards in the online learning problem: specifically, if the optimizer plays an action $1 \leq i \leq T$ and the learner plays action $j \in [N]$, the learner will receive the reward they would have received if they played $j$ in the $i$th round of the online learning instance, namely $u_{L}(i, j) = r^{i}_{j}$. Note that this has the property that if the optimizer plays action $t$ in round $t$, the learning algorithm $\cA$ will see exactly the learning instance mentioned above and therefore play $\beta^t$ in each round $t$. 

The more interesting aspect of constructing this game is selecting payoffs for the optimizer. We do this as follows. If the optimizer plays action $i$ and the learner plays action $j$, we set $u_{O}(i, j) = \frac{1}{2}(r^{i}_{\pi(j)} - r^{i}_{j})$ (note that since $r^{i} \in [-1, 1]^N$, this scaling ensures that $|u_{O}(i, j)| \leq 1$). This lets us rewrite \eqref{eq:hiswap} in the form 

\begin{equation}\label{eq:hiswap2}
    \sum_{t=1}^{T} u_O(t, \beta^t) = \frac{R}{2}. 
\end{equation}

But note that the LHS of \eqref{eq:hiswap2} is exactly the total payoff the optimizer receives if they play action $t$ in round $T$ against a learner running $\cA$. Therefore the optimizer can obtain a total reward of $R/2$ against such a learner. 

We now argue that the Stackelberg value of this game is at most $0$. To do this, imagine that in a single-shot Stackelberg instance of $G$, the optimizer plays a mixed strategy $\alpha \in \Delta([M])$ and the learner best-responds by playing $j \in [N]$. Then note that
\[
u_{O}(\alpha, j) = \frac{1}{2} \sum_{i=1}^{M} \alpha_i \left(r^{i}_{\pi(j)} - r^{i}_j\right) = \frac{1}{2}\left(u_{L}(\alpha, \pi(j)) - u_{L}(\alpha, j)\right).
\]

But since $j$ is a best response to $\alpha$ for the learner, we must have that $u_{L}(\alpha, j) \geq u_{L}(\alpha, \pi(j))$; it follows that $u_{O}(\alpha, j) \leq 0$, and therefore $\Val(G) \leq 0$. This completes our proof of the existence of a game (albeit one with many actions for the optimizer) where the optimizer can get at least $R/2$ more than the Stackelberg value of the game.

We now show how to construct a game $G'$ with the same property, but where the optimizer only has $M = 2^N$ actions. To do this, observe that if you fix an action $i \in [M]$ for the optimizer, both the optimizer's payoff $u_{O}(i, \cdot)$ and learner's payoff $u_{L}(i, \cdot)$ are linear functions of $r^{i}$. This motivates the following construction. For each $1 \leq i \leq M$ let $s^i$ be the $i$th element of $S = \{-1, 1\}^N$ (for some arbitrary labelling of the $M$ elements of $S$), and let:
\begin{eqnarray*}
u_{L}(i, j) &=& s^{i}_j \\
u_{O}(i, j) &=& \frac{1}{2}(s^i_{\pi(j)} - s^{i}_j).
\end{eqnarray*}

Since $S$ contains the vertices of $[-1, 1]^N$, by Caratheodory's theorem, for each $r^{t} \in [-1, 1]^N$, there exists an $\alpha^{t} \in \Delta([M])$ such that $r^{t} = \sum_{i=1}^{M}\alpha^{t}_i s^{i}$. In particular, this implies that $u_L(\alpha^{t}, j) = r^{t}_j$ and $u_{O}(\alpha^{t}, j) = (r^{t}_{\pi(j)} - r^{t}_j)/2$, so by playing the sequence of actions $\alpha^t$, the optimizer can still guarantee total reward $R/2$. Moreover, the Stackelberg value of $G'$ is still $0$, by the same logic as before. 
\end{proof}

\begin{remark}
There is a sense in which the exponential dependence of $M$ on $N$ in Lemma \ref{lem:standard} is unnecessary. In particular, if there exists a bad instance for $\cA$ where all the rewards lie in $[-1/N, 1/N]^N$, then we can conduct the same construction at the end of Lemma \ref{lem:standard}, but with the set of $2N$ vectors $S = \{e_1, -e_1, e_2, -e_2, \dots, e_{N}, -e_{N}\}$ (in general, it is only necessary that it is possible to write each reward vector $r^{t}$ as a convex combination of elements of $S$). 
\end{remark}

We now apply Lemma \ref{lem:standard} to show that if the learner is \textit{not} running a low-swap-regret algorithm, there is a game $G$ where an optimizer can get $\Omega(T)$ more than the Stackelberg value.

\begin{theorem}\label{thm:main_standard}
If $\cA$ is not a low-swap-regret learning algorithm, then there exists a game $G$ where if an optimizer plays $T$ rounds of $G$ against a learner running $\cA$, the optimizer can get reward at least $\Val(G)T + \Omega(T)$ for infinitely many values of $T$. 
\end{theorem}

\section{Bayesian games}\label{sec:bayesian}

We now begin our exploration of Bayesian games. The model for online learning introduced in Section \ref{sec:intro_learning} covers algorithms a learner can use to play standard games. In order to play a Bayesian game, the learner needs to use an algorithm for \textit{online contextual learning}. 

Our model for online contextual learning will be similar to the model for online learning presented in Section \ref{sec:intro_learning} (in that it will be fractional, full-information, and deterministic), with some differences analogous to the difference between standard games and Bayesian games. Specifically, in online contextual learning:

\begin{itemize}[leftmargin=*,noitemsep,nolistsep]
    \item There is a publicly known distribution $\cD$ over $C$ contexts, where context $c \in [C]$ occurs with probability $p_{c}$.
    \item Instead of picking a sequence of $T$ reward vectors in $[0, 1]^{N}$, the adversary picks $T$ reward vectors in $r^{t} \in [0, 1]^{N \times C}$, where $r^{t}_{i, c}$ represents the reward if the learner picks action $i \in [N]$ in context $c \in [C]$. 
    \item Instead of selecting a single mixed action in round $t$, the learner instead picks a function $\beta^{t}\colon [C] \rightarrow \Delta([N])$ mapping contexts to distributions over actions. 
    
    \item In round $t$, the learner receives reward $\sum_{c=1}^{C}\sum_{j=1}^{N}p_{c}\beta^{t}(c)_{j} r^t_{j, c}$. The learner wishes to maximize their total reward over all $T$ rounds.
\end{itemize}

It is straightforward to define a notion of external regret for online contextual learning. If we let

\begin{align*}
\Reg = \left(\max_{f^{*} : [C] \rightarrow [N]} \sum_{t=1}^{T} \sum_{c=1}^{C} p_{c}r^{t}_{f^{*}(c), c}\right) - \left(\sum_{t=1}^{T}\sum_{c=1}^{C}\sum_{j=1}^{N}p_{c}\beta^{t}(c)_{j} r^t_{j, c} \right),
\end{align*}

\noindent
then $\Reg$ represents the difference in utilities on this problem instance between the learning algorithm and the learner who in each context $c$ plays the best-in-hindsight action $f^*(c)$ for this context. It is similarly straightforward to construct algorithms for online contextual learning which incur external regret at most $O(\sqrt{T\log N})$ in the above setting (for example, one can simply run a low-regret online learning algorithm per context). 

Interestingly, it is much less clear what the correct analogue of swap regret for online contextual learning is. Many obvious guesses (such as the notion of regret obtained by running a low-swap-regret algorithm per context) turn out to be ``incorrect'', in that they do not allow us to prove analogues of Theorem \ref{thm:main_standard} for Bayesian games (we explore these in more detail in Section \ref{sec:bayesian_algs}). Ultimately, in this section we will define two notions of swap-regret with the following guarantees:

\begin{itemize}[leftmargin=*,noitemsep,nolistsep]
    \item \textbf{Polytope swap regret}: If a contextual learning algorithm $\cA$ has low polytope swap regret, then an optimizer can get at most $\Val(G)T + o(T)$ reward when playing a Bayesian game $G$ for $T$ rounds against a learner running $\cA$. 
    \item \textbf{Linear swap regret}: If a contextual learning algorithm $\cA$ \textit{does not have} low linear swap regret, there exists a Bayesian game $G$ where if an optimizer plays $T$ rounds of $G$ against a learner using $\cA$, the optimizer can get at least $\Val(G)T + \Omega(T)$ reward (for infinitely many values of $T$). 
\end{itemize}

\subsection{Polytope games and polytope learning}

Instead of defining these notions of swap regret directly for Bayesian games, we will find it convenient to define a more general class of games and learning algorithms that  generalizes both forms of games (standard and Bayesian) and learning (regular and contextual) that we have considered so far. We call this class of games \textit{polytope games}, and the corresponding variety of learning \textit{polytope learning}. Polytope learning will turn out to be essentially equivalent to online linear optimization (the major difference being that we restrict the action set to be a polytope as opposed to an arbitrary convex set), but we refer to it this way to emphasize the connection to polytope games. 

In a polytope game $G$, the learner must select a point $x$ belonging to some bounded polytope $\cP \subseteq [-1, 1]^d$. Simultaneously, the optimizer must select a point $q = (r, s)$ belonging to some polytope $\cQ \subseteq [-1, 1]^d \times [-1, 1]^{d}$ (i.e, both $r$ and $s$ are $d$-dimensional vectors). The optimizer then receives utility $\langle s, x\rangle$, and the learner receives utility $\langle r, x \rangle$. As with standard games and Bayesian games, we can easily define the Stackelberg value $\Val(G)$ of this game to be the maximum value an optimizer can guarantee by playing a fixed\footnote{Note that since the reward function is linear and the action spaces of both the optimizer and learner are convex, there is no need to consider mixed actions -- for either player, any mixed action is equivalent to some single action.} action $(r, s) \in \cQ$ with the learner best responding.

%\footnote{YM: why fixed action and not a distribution over actions?! MM: I think these fixed actions are distributions for us! JS: yes, this is good to clarify -- I added a footnote a little earlier.}

To play a repeated polytope game, a learner can run a polytope learning algorithm. An instance of polytope learning is specified by a polytope $\cP$ and a sequence of $T$ $d$-dimensional reward vectors $r^{1}, r^{2}, \dots, r^{T} \in [-1, 1]^d$. At the beginning of round $t$, the learner must select a point $x^{t} \in \cP$; the learner then receives reward $\langle r^{t}, x^{t} \rangle$. The goal of the learner is to maximize their total reward. As with our previous learners, we will restrict our attention to deterministic algorithms (i.e., algorithms $\cA$ that choose $x^{t}$ deterministically as a function of $r^{1}, r^{2}, \dots, r^{t-1}$).

Here are some examples of polytope games and polytope learning:

% \jon{todo: consider making the reduction below a theorem}
%\begin{itemize}
    %\item 
    \paragraph{1) Standard games and online learning.} Given a standard game $G$, let $\cP = \Delta([N])$ and let $\cQ = \conv(\{(r_i, s_i)\}_{i=1}^{M})$, where for each $i \in [M]$
    \begin{eqnarray*}
    r_{i} &=& (u_{L}(i, 1), u_{L}(i, 2), \dots, u_{L}(i, N)) \in \mathbb{R}^N\\
    s_{i} &=& (u_{O}(i, 1), u_{O}(i, 2), \dots, u_{O}(i, N)) \in \mathbb{R}^N.
    \end{eqnarray*}
    
    Then the polytope game $G'$ defined by $\cP$ and $\cQ$ is ``equivalent'' to the standard game $G$ in the following sense: the map $f$ which sends $\alpha \in \Delta([M])$ to the point $q = \sum_{i=1}^{M} \alpha_{i}(r_{i}, s_{i}) \in \cQ$ and the identity map $g$ which sends $\beta \in \Delta([N])$ to $\beta \in \cP$ together have the property that $u_{O}(\alpha, \beta) = u_{O}(f(\alpha), g(\beta))$ and $u_{L}(\alpha, \beta) = u_{L}(f(\alpha), g(\beta))$; moreover, $f$ and $g$ are both surjective onto $\cP$ and $\cQ$ respectively. Intuitively, we can translate any strategy profile in $G$ to an ``equivalent'' strategy profile in $G'$ and vice versa. Note that there may be multiple strategy profiles in $G$ that map to the same strategy profile in $G'$; this corresponds to the fact that it is possible for two different mixed strategies for the learner to always result in the same payoffs. 
    
    % \footnote{YM: $g$ is simply the identity here, right?! JS: Yes. I made an edit to explicitly point this out, but I feel this paragraph is probably still a little confusing... feel free to edit if you have ideas for making it clearer.} 
    
    Similarly, polytope learning over the simplex $\cP = \Delta([N])$ is equivalent to online learning. Here the reduction is immediate -- the online learning problem in Section \ref{sec:intro_learning} is exactly polytope learning with $\cP = \Delta([N])$. 
    
    %\item 
    \paragraph{2) Bayesian games and contextual learning.} Given a Bayesian game $G$, let $\cP = \Delta([N])^{C} \subseteq \mathbb{R}^{N \times C}$ (one can think of $\cP$ as the convex hull of all $C$-by-$N$ stochastic matrices), and let $\cQ = \conv(\{(r_i, s_i)\}_{i=1}^{M})$, where $s_i \in \mathbb{R}^{N \times C}$ is defined via $s_{i, j, c} = p_{c}u_{O}(i, j, c)$ and $r_{i} \in \mathbb{R}^{N \times C}$ is defined via $r_{i, j, c} = p_{c}u_{L}(i, j, c)$. Then the polytope game $G'$ defined by $\cP$ and $\cQ$ is equivalent to the Bayesian game $G$ (in the same sense and for the same reason as above).
    
    Similarly, polytope learning over the polytope $\cP = \Delta([N])^{C}$ is the contextual learning problem we defined at the beginning of this section (with the subtle difference that the reward vectors $r^{t}_{i, c}$ should be scaled by the probabilities $p_{c}$). 

In general, we will want to fix a polytope $\cP$ (e.g. the Bayesian game polytope) and consider the class of polytope games where the learner's actions belong to $\cP$. We call such games \textit{$\cP$-games}; note that such games are entirely specified by the optimizer's action set $\cQ$. Similarly, when discussing polytope learning, we will want to fix a polytope $\cP$ and consider the class of polytope learning instances / algorithms where the learner's actions must belong to $\cP$. We call this problem \textit{$\cP$-learning} for short, and such algorithms \textit{$\cP$-learning algorithms}.

\subsection{Linear swap regret}
\label{sec:linswapreg}
We begin by defining a variant of swap regret for polytope learning that we call \textit{linear swap regret}. In this variant of swap regret, a learner compares their utility to the utility they would have received if they applied a static linear transformation to each of their actions.

Formally, let $\cP$ be a polytope and consider an instance of the  $\cP$-learning problem where the rewards are $r^{1}, r^{2}, \dots, r^{T}$ and the actions of the learner are given by $x^{1}, x^{2}, \dots, x^{T}$. Then the linear swap regret of the learner on this instance is given by

$$\LinSwapReg = \max_{M \in \cM(\cP)}\sum_{t=1}^{T} \langle r^{t}, Mx^{t} \rangle  - \sum_{t=1}^{T} \langle r^{t}, x^{t} \rangle,$$

\noindent
where $\cM(\cP)$ is the set of all linear transformations $M: \mathbb{R}^{d} \rightarrow \mathbb{R}^{d}$ that satisfy $Mx \in \cP$ for all $x \in \cP$ (i.e., the set of linear transformations that are contractions of $\cP$). 

As with other forms of regret, we say that a polytope learning algorithm $\cA$ has low linear swap regret if it sustains $o(T)$ linear swap regret on any problem instance with $T$ rounds. We now show that in order for the optimizer to get no more than the Stackelberg value, it is necessary for the learner to run a low linear swap regret algorithm. As in Section \ref{sec:standard}, we begin by showing this is true on a per instance level.

\begin{lemma}\label{lem:linear}
Fix a polytope $\cP \subseteq [-1, 1]^d$. Let $\cA$ be a $\cP$-learning algorithm which has linear swap regret $R$ on some problem instance. Then there exists a $\cP$-game such that if an optimizer plays $T$ rounds of $G$ against a learner running $\cA$, the optimizer can receive a total reward of $\Val(G)T + R/(\lambda + 1)$, where $\lambda = \max_{M \in \cM(\cP)} ||M||_{1}$.
\end{lemma}

When $\cP = \Delta([N])$, linear swap regret reduces to the ordinary notion of swap regret. Indeed, the set $\cM(\Delta([N]))$ of linear contractions of $\Delta([N])$ is exactly the set of $N$-by-$N$ stochastic matrices, and the extreme points of $\cM(\Delta([N]))$ are the $N$-by-$N$ 0/1-matrices which contain exactly one $1$ in each row. These matrices correspond to (and act the same way on $\pi$) as the $N^N$ swap functions $\pi: [N] \rightarrow [N]$. 

As with standard games, we can apply Lemma \ref{lem:linear} to show that if $\cA$ does not have low linear swap regret, then it is possible for an optimizer to get strictly more than their Stackelberg value by playing against a learner running $\cA$ in a fixed game.

\begin{theorem}\label{thm:linear}
Fix a polytope $\cP \subseteq [-1, 1]^d$. If a $\cP$-learning algorithm $\cA$ \textit{does not have} low linear swap regret, then there exists a $\cP$-game $G$ where if an optimizer plays $T$ rounds of $G$ against a learner running $\cA$, the optimizer can get at least $\Val(G)T + \Omega(T)$ reward (for infinitely many values of $T$). 
\end{theorem}

\subsection{Polytope swap regret}
\label{sec:psr}
We now define a second notion of swap regret for $\cP$-learning algorithms which we call \textit{polytope swap regret}. Whereas having low linear swap regret is a necessary condition to guarantee that the optimizer receives at most $\Val(G)T + o(T)$ utility, we will show that having low polytope swap regret is a sufficient condition for the same guarantee.

The intuition behind polytope swap regret is that we want to compete against an arbitrary swap function $\pi$ on the vertices of $\cP$ (i.e., a function that maps each vertex of $\cP$ to some other vertex). This makes sense if the learner only ever plays actions which are vertices of $\cP$, but it is less clear how $\pi$ should act on an interior point $x$ of $\cP$. One way to define an action of $\pi$ on $x$ is to write $x$ as a convex combination of vertices, use $\pi$ to map each of these vertices to (possibly) new vertices, and take the corresponding convex combination of these new vertices to obtain a new point $x'$. This works, but in many cases there will many ways to write $x$ as a convex combination of the vertices of $\cP$. When computing polytope swap regret, we will choose the best (regret-minimizing) decompositions of all actions $x^{i}$ in our given problem instance that minimizes the worst-case regret for the worst possible swap function $\pi$.

Formally, given a polytope $\cP$, let $\cV(\cP)$ be the set of vertices of $\cP$, and let $V = |\cV(\cP)|$ be the number of vertices of $\cP$. We say that $\rho \in \Delta(\cV(\cP))$ is a \textit{vertex decomposition} of a point $x \in \cP$ if $\sum_{v \in \cV(\cP)} \rho_{v}v = x$; likewise, given a vertex decomposition $\rho$, we will let $\overline{\rho} = \sum_{v \in \cV(\cP)} \rho_{v}v$ denote the point in $\cP$ for which $\rho$ is a vertex decomposition. 

We call functions $\pi:\cV(\cP) \rightarrow \cV(\cP)$ that map vertices of $\cP$ to vertices of $\cP$ \textit{vertex swap functions}. We extend $\pi$ to act on vertex decompositions by letting $\pi(\rho)_{v} = \sum_{v' \in \pi^{-1}(v)} \rho_{v'}$; note that under this definition, $\overline{\pi(\rho)} = \sum_{v \in \cV(\cP)}\rho_{v}\pi(v)$.

Finally, consider an instance of the $\cP$-learning problem with reward vectors $r^{1}, r^{2}, \dots, r^{T} \in [-1, 1]^d$, where a learner running algorithm $\cA$ plays actions $x^{1}, \dots, x^{T} \in \cP$. We define the polytope swap regret of $\cA$ on this instance as

$$\PolySwapReg = \left(\min_{\rho^{t}\,|\,\overline{\rho^{t}} = x^{t}} \max_{\pi : \cV(\cP) \rightarrow \cV(\cP)} \sum_{t=1}^{T} \left\langle r^{t}, \overline{\pi(\rho^{t})} \right\rangle \right) - \sum_{t=1}^{T} \langle r^{t}, x^{t} \rangle.$$

Here the outer minimum is over all sequences of vertex decompositions $\rho^{t}$ of the actions $x^{t}$, and the inner maximum is over all vertex swap functions $\pi$. An alternate way of thinking about this benchmark is in the form of a zero-sum game: the learner, after they have played all their actions $x^{t}$ (but before they see $\pi$) chooses a vertex decomposition $\rho^{t}$ for each of their actions. The adversary then observes these vertex decompositions and responds with the vertex swap function $\pi$ which maximizes the counterfactual utility of the transformed action sequence obtained by applying $\pi$ to each $\rho^{t}$. The learner wishes to choose their original decompositions to minimize this maximum counterfactual utility.

As always, we say the $\cP$-learning algorithm $\cA$ has low polytope swap regret if it incurs at most $o(T)$ polytope swap regret for all $\cP$-learning instances.

\begin{theorem}\label{thm:poly_regret}
Let $\cA$ be a $\cP$-learning algorithm with low polytope swap regret. Then if an optimizer plays $T$ rounds of a $\cP$-game $G$ against a learner running $\cA$, the optimizer can get at most $\Val(G)T + o(T)$ reward. 
\end{theorem}

\subsection{Separating linear and polytope swap regret in Bayesian games}
We establish separation of linear and polytope swap regrets in Appendix~\ref{sec:linpolysep} and pose an open question about polytope swap regret (Question~\ref{qn:polyneces}).

\bibliography{sbg}

\begin{thebibliography}{23}
\providecommand{\natexlab}[1]{#1}
\providecommand{\url}[1]{\texttt{#1}}
\expandafter\ifx\csname urlstyle\endcsname\relax
  \providecommand{\doi}[1]{doi: #1}\else
  \providecommand{\doi}{doi: \begingroup \urlstyle{rm}\Url}\fi

\bibitem[Agrawal et~al.(2018)Agrawal, Daskalakis, Mirrokni, and Sivan]{ADMS18}
Shipra Agrawal, Constantinos Daskalakis, Vahab~S. Mirrokni, and Balasubramanian
  Sivan.
\newblock Robust repeated auctions under heterogeneous buyer behavior.
\newblock In \emph{Proceedings of the 2018 {ACM} Conference on Economics and
  Computation, Ithaca, NY, USA, June 18-22, 2018}, page 171, 2018.

\bibitem[Aumann(1974)]{Aumann74}
Robert~J. Aumann.
\newblock Subjectivity and correlation in randomized strategies.
\newblock \emph{Journal of Mathematical Economics}, 1\penalty0 (1):\penalty0 67
  -- 96, 1974.
\newblock ISSN 0304-4068.

\bibitem[Blum and Mansour(2005)]{BM05}
Avrim Blum and Yishay Mansour.
\newblock From external to internal regret.
\newblock In Peter Auer and Ron Meir, editors, \emph{Learning Theory}, 2005.

\bibitem[Blum and Mansour(2007)]{BlumMansour2007}
Avrim Blum and Yishay Mansour.
\newblock From external to internal regret.
\newblock \emph{Journal of Machine Learning Research}, 8:\penalty0 1307--1324,
  2007.

\bibitem[Braverman et~al.(2018)Braverman, Mao, Schneider, and
  Weinberg]{braverman2018selling}
Mark Braverman, Jieming Mao, Jon Schneider, and Matt Weinberg.
\newblock Selling to a no-regret buyer.
\newblock In \emph{Proceedings of the 2018 ACM Conference on Economics and
  Computation}, pages 523--538, 2018.

\bibitem[Cesa-Bianchi and Lugosi(2003)]{CL03}
Nicol{\`o} Cesa-Bianchi and G{\'a}bor Lugosi.
\newblock Potential-based algorithms in on-line prediction and game theory.
\newblock \emph{Machine Learning}, 51\penalty0 (3):\penalty0 239--261, Jun
  2003.

\bibitem[Cesa-Bianchi et~al.(1997)Cesa-Bianchi, Freund, Haussler, Helmbold,
  Schapire, and Warmuth]{CFHHSW97}
Nicol\`{o} Cesa-Bianchi, Yoav Freund, David Haussler, David~P. Helmbold,
  Robert~E. Schapire, and Manfred~K. Warmuth.
\newblock How to use expert advice.
\newblock \emph{J. ACM}, 44\penalty0 (3):\penalty0 427--485, May 1997.
\newblock ISSN 0004-5411.

\bibitem[Chleb{\'\i}k and Chleb{\'\i}kov{\'a}(2008)]{chlebik2008approximation}
Miroslav Chleb{\'\i}k and Janka Chleb{\'\i}kov{\'a}.
\newblock Approximation hardness of dominating set problems in bounded degree
  graphs.
\newblock \emph{Information and Computation}, 206\penalty0 (11):\penalty0
  1264--1275, 2008.

\bibitem[Conitzer(2016)]{conitzer2016stackelberg}
Vincent Conitzer.
\newblock On stackelberg mixed strategies.
\newblock \emph{Synthese}, 193\penalty0 (3):\penalty0 689--703, 2016.

\bibitem[Deng et~al.(2019)Deng, Schneider, and Sivan]{deng2019strategizing}
Yuan Deng, Jon Schneider, and Balasubramanian Sivan.
\newblock Strategizing against no-regret learners.
\newblock \emph{Advances in neural information processing systems}, 32, 2019.

\bibitem[Foster and Vohra(1993)]{FV93}
Dean~P. Foster and Rakesh~V. Vohra.
\newblock A randomization rule for selecting forecasts.
\newblock \emph{Operations Research}, 41\penalty0 (4):\penalty0 704--709, 1993.

\bibitem[Foster and Vohra(1997)]{FV97}
Dean~P. Foster and Rakesh~V. Vohra.
\newblock Calibrated learning and correlated equilibrium.
\newblock \emph{Games and Economic Behavior}, 21\penalty0 (1):\penalty0 40 --
  55, 1997.

\bibitem[Foster and Vohra(1998)]{FV98}
Dean~P. Foster and Rakesh~V. Vohra.
\newblock {Asymptotic calibration}.
\newblock \emph{Biometrika}, 85\penalty0 (2):\penalty0 379--390, 06 1998.

\bibitem[Foster and Vohra(1999)]{FV99}
Dean~P. Foster and Rakesh~V. Vohra.
\newblock Regret in the on-line decision problem.
\newblock \emph{Games and Economic Behavior}, 29\penalty0 (1):\penalty0 7 --
  35, 1999.

\bibitem[Freund and Schapire(1997{\natexlab{a}})]{FS97}
Yoav Freund and Robert~E. Schapire.
\newblock A decision-theoretic generalization of on-line learning and an
  application to boosting.
\newblock \emph{Journal of Computer and System Sciences}, 55\penalty0
  (1):\penalty0 119 -- 139, 1997{\natexlab{a}}.

\bibitem[Freund and Schapire(1997{\natexlab{b}})]{FreundSchapire1997}
Yoav Freund and Robert~E. Schapire.
\newblock A decision-theoretic generalization of on-line learning and an
  application to boosting.
\newblock \emph{J. Comput. Syst. Sci.}, 55\penalty0 (1):\penalty0 119--139,
  1997{\natexlab{b}}.

\bibitem[Freund and Schapire(1999)]{FS99}
Yoav Freund and Robert~E. Schapire.
\newblock Adaptive game playing using multiplicative weights.
\newblock \emph{Games and Economic Behavior}, 29\penalty0 (1):\penalty0 79 --
  103, 1999.

\bibitem[Guruganesh et~al.(2021)Guruganesh, Schneider, and
  Wang]{guruganesh2021contracts}
Guru Guruganesh, Jon Schneider, and Joshua~R Wang.
\newblock Contracts under moral hazard and adverse selection.
\newblock In \emph{Proceedings of the 22nd ACM Conference on Economics and
  Computation}, pages 563--582, 2021.

\bibitem[Hannan(1957)]{Hannan57}
James Hannan.
\newblock Approximation to bayes risk in repeated plays.
\newblock \emph{Contributions to the Theory of Games}, 3:\penalty0 97--139,
  1957.

\bibitem[Hart and Mas-Colell(2000)]{HM00}
Sergiu Hart and Andreu Mas-Colell.
\newblock A simple adaptive procedure leading to correlated equilibrium.
\newblock \emph{Econometrica}, 68\penalty0 (5):\penalty0 1127--1150, 2000.

\bibitem[Littlestone and Warmuth(1994{\natexlab{a}})]{LW94}
N.~Littlestone and M.K. Warmuth.
\newblock The weighted majority algorithm.
\newblock \emph{Information and Computation}, 108\penalty0 (2):\penalty0 212 --
  261, 1994{\natexlab{a}}.

\bibitem[Littlestone and Warmuth(1994{\natexlab{b}})]{LittlestoneWarmuth1994}
Nick Littlestone and Manfred~K Warmuth.
\newblock The weighted majority algorithm.
\newblock \emph{Information and computation}, 108\penalty0 (2):\penalty0
  212--261, 1994{\natexlab{b}}.

\bibitem[Von~Stengel and Zamir(2010)]{von2010leadership}
Bernhard Von~Stengel and Shmuel Zamir.
\newblock Leadership games with convex strategy sets.
\newblock \emph{Games and Economic Behavior}, 69\penalty0 (2):\penalty0
  446--457, 2010.

\end{thebibliography}

\newpage
\appendix

\section{Related Work}
\label{sec:app_related}
There is a very large amount of literature on the outcome of interaction between strategic agents in single-shot and repeated games, offering us reasonably complete picture of the landscape here. Likewise, when learning agents interact repeatedly in a game, we have a good understanding of what equilibria they lead to, depending on the learning algorithms employed by the learners (more on this below). However, the nature of this interaction between a strategic agent and a learner is much less studied, and has gained momentum only in the last few years. The work closest to ours in this space is that of~\cite{deng2019strategizing} who initiate the study of this optimizer-learner interaction and draw some very interesting conclusions detailed earlier. The recent work of~\cite{braverman2018selling} is also quite close to ours. 
They study the specific 2-player Bayesian game of an auction between a single seller and single buyer. The seller’s choice of the auction to run represents his action, and the buyer’s bid represents her
action. Our work generalizes both~\cite{deng2019strategizing} and~\cite{braverman2018selling} by studying general Bayesian games, and also addresses questions beyond what was asked in those works. Both~\citeauthor{deng2019strategizing} and~\citeauthor{braverman2018selling} show that regardless of the specific algorithm used by the learner (buyer), as long as the buyer plays a no-regret learning algorithm, the optimizer (seller) can always earn at least the Stackelberg value in a single
shot game. Our Lemma~\ref{lem:stack_achieve} in Appendix is a direct generalization of both these results to arbitrary Bayesian games without any structure. Both~\citeauthor{deng2019strategizing} and~\citeauthor{braverman2018selling} show that there exist no-regret strategies for the learner (buyer)
that guarantee that the optimizer (seller) cannot get anything better than the single-shot optimal payoff. Our Theorem~\ref{thm:poly_regret} on polytope swap regret is a directly generalization of both these results showing that when learner has a low polytope swap regret, the optimizer cannot earn much more than the Bayesian Stackelberg value of the game. Neither~\citeauthor{deng2019strategizing} nor~\citeauthor{braverman2018selling} provide necessary and sufficient conditions for the learner to cap the optimizer payoff at Stackelberg value. We provide such necessary and sufficient conditions for standard games in\footnote{Theorem~\ref{thm:main_standard} establishes necessity, as sufficiency is already known from~\cite{deng2019strategizing}.} Theorem~\ref{thm:main_standard}. We provide a partial answer for Bayesian games, where we get a sufficient condition that we conjecture to be necessary as well.

To discuss literature on the outcome of learning agents interacting with each other, we begin with the different notions of regret. The usual notion of regret, without the swap qualification, is often referred to as external-regret (see~\cite{Hannan57},~\cite{FV93},~\cite{LW94},~\cite{FS97},~\cite{FS99},~\cite{CFHHSW97}). There is a stronger notion of regret called internal regret that was defined earlier in~\cite{FV98}, which allows all occurrences of a given action $x$ to be replaced by another action $y$. Many no-internal-regret algorithms have been designed (see for example~\cite{HM00},~\cite{FV97,FV98,FV99},~\cite{CL03}). The still stronger notion of swap regret was introduced in~\cite{BM05}, and it allows one to simultaneously swap several pairs of actions.~\citeauthor{BM05} show how to efficiently convert a no-regret algorithm to a no-swap-regret algorithm. One of the reasons behind the importance of internal and swap regret is their close connection to the central notion of correlated equilibrium introduced by~\cite{Aumann74}. In a general $n$ players game, a distribution over action profiles of all the players is a correlated equilibrium if every player has zero internal regret. When all players use algorithms with no-internal-regret guarantees, the time averaged strategies of the players converges to a correlated equilibrium (see~\cite{HM00}). When all players simply use algorithms with no-external-regret guarantees, the time averaged strategies of the players converges to the weaker notion of coarse correlated equilibrium. When the game is a zero-sum game, the time-averaged strategies of players employing no-external-regret dynamics converges to the Nash equilbrium of the game.

On the special $2$-player game of selling to a buyer in an auction,~\cite{ADMS18} study a setting similar to~\citeauthor{braverman2018selling} but also consider other types of buyer behavior apart from learning, and show to how to robustly optimize against various buyer strategies in an auction.

\section{Separating linear and polytope swap regrets}
\label{sec:linpolysep}
In Sections~\ref{sec:linswapreg} and~\ref{sec:psr}, we presented two different definitions of swap regret -- linear swap regret and polytope swap regret -- for the very general setting of polytope learning. Together, they form necessary and sufficient conditions for which learning algorithms are robust to strategic behavior in polytope games: algorithms with low polytope swap regret are always robust, whereas algorithms with high linear swap regret are not robust. Interestingly, when $\cP$ is the simplex $\Delta([N])$ both these notions of swap regret are equal, and they both reduce to the ordinary notion of swap regret in online learning.  

We now return our attention to the setting of Bayesian games and contextual learning, where the relevant polytope $\cP = \Delta([N])^{C}$ is the product of $C$ $N$-simplices. Just as swap regret is the ``correct'' regret definition for characterizing robustness in standard games, we want to understand what is the correct definition of regret for characterizing robustness in Bayesian games. Is it polytope swap regret, or linear swap regret, or some as-yet-undefined notion that lies between these two regret measures?

We do not yet know the answer to this question, but we provide some partial progress towards resolving it. To begin, we show that unlike for standard games, for the Bayesian game polytope polytope swap regret and linear swap regret can differ significantly (even when $K=C=2$).

\begin{theorem} \label{thm:poly_lin_sep}
Let $\cP = \Delta([N])^{C}$. There exists a $\cP$-learning instance with $K=C=2$ where $\PolySwapReg = \Omega(T)$ and where $\LinSwapReg = 0$.
\end{theorem}

This separation implies that at least one of polytope swap regret and linear swap regret do not tightly characterize robustness against strategic behavior in Bayesian games. It turns out (as we discovered through computational search) that linear swap regret is not tight: it is possible to extend the example in Theorem \ref{thm:poly_lin_sep} to a Bayesian game where the optimizer can get $\Omega(T)$ more than their Stackelberg value, despite the learner running a low linear regret learning algorithm.

\begin{theorem} \label{thm:poly_lin_sep_bayesian}
Let $\cP = \Delta([N])^{C}$. There exists a $\cP$-game (i.e., Bayesian game) $G$ with $K=C=2$ and a low linear regret $\cP$-learning algorithm (i.e., contextual learner) $\cA$ where if an optimizer plays $T$ rounds of $G$ against a learner running $\cA$, the optimizer can get reward at least $V(G)T + \Omega(T)$.
\end{theorem}
its
On the other hand, we are unable to find any example of a contextual learning instance where a learner incurs high polytope swap regret that can be extended to a game where the optimizer \textit{cannot} get more than their Stackelberg value. We conjecture that polytope swap regret is the ``correct'' notion of regret, both for Bayesian games and more generally for polytope games. Specifically, we pose the following open question:

\begin{question}
\label{qn:polyneces}
Let $\cA$ be a contextual learning algorithm which does not have low polytope swap regret. Is it possible for an optimizer to get at least $\Val(G)T + \Omega(T)$ reward when playing $T$ rounds of a Bayesian game $G$ against an optimizer running $\cA$?
\end{question}

\section{Algorithmic considerations}
\label{sec:algcons}
\subsection{A low polytope swap regret algorithm} \label{sec:lowpolyalg}

Thus far, we have focused on characterizing relevant definitions of regret for the learner without actually discussing how to construct learning algorithms that minimize these forms of regret. We remedy this here by showing how to convert a traditional low-swap-regret algorithm to an algorithm for polytope learning which obtains $o(T)$ polytope swap regret (and hence $o(T)$ linear swap regret as well). 

Let $\cA$ be a low-swap-regret algorithm for online swap regret (e.g., the one presented in \cite{BlumMansour2007}). The idea is simple: if $\cA'$ is faced with a $\cP$-learning instance, $\cA'$ begins by initializing a copy of $\cA$ with one action for each vertex $v \in \cV(\cP)$. At the beginning of each round $t$, $\cA'$ queries $\cA$ for a mixed strategy $\beta^t \in \Delta(\cV(\cP))$ and then plays the corresponding convex combination of the vertices of $\cP$, the point $x^{t} = \sum_{v \in \cV(\cP)} \beta^t_{v} v \in \cP$. Then $\cA'$ will compute the reward $\langle r^{t}, v\rangle$ for each vertex $v$ and pass this collection of rewards to $\cA$.

We show that swap regret guarantees for $\cA$ directly translate to polytope swap regret guarantees for $\cA'$.

\begin{theorem}
Let $\cA$ be an online learning algorithm which incurs swap regret at most $R(N, T)$ on instances with $N$ actions and $T$ rounds. Then $\cA'$ incurs polytope swap regret at most $R(|\cV(\cP)|, T)$ on $\cP$-learning instances over $T$ rounds.
\end{theorem}
\begin{proof}
Let $\beta^{t} \in \Delta(\cV(\cP))$ be the mixed strategy output by $\cA$ in round $t$, and let $x^t = \sum_{v \in \cV(\cP)} \beta^t_{v} v$ be the point in $\cP$ played by $\cA'$ in round $t$. Note that $\beta^{t}$ is a vertex partition of $x^{t}$. If we use these vertex partitions in the definition of polytope swap regret, we find that

$$\PolySwapReg \leq \left(\max_{\pi : \cV(\cP) \rightarrow \cV(\cP)} \sum_{t=1}^{T} \left\langle r^{t}, \overline{\pi(\beta^{t})} \right\rangle \right) - \sum_{t=1}^{T} \langle r^{t}, x^{t} \rangle.$$

If we rewrite the RHS of the above equation (decomposing by vertices in the same manner as in the proof of Theorem \ref{thm:poly_regret}), we have that

$$\PolySwapReg \leq \max_{\pi : \cV(\cP) \rightarrow \cV(\cP)}\sum_{v \in \cV(\cP)} \sum_{t=1}^{T} \beta^{t}_{v}\left\langle r^{t}, \pi(v) - v \right\rangle.$$

\noindent
But the right hand side of the above equation is exactly the value of $\SwapReg$ faced by $\cA$. This regret in turn is at most $R(|\cV(\cP)|, T)$ by the guarantees of $\cA$.
\end{proof}

\subsection{Efficient learning algorithms for Bayesian games}\label{sec:bayesian_algs}

The reduction in the previous section produces an algorithm incurs at most $O(\sqrt{TV\log V})$ polytope swap regret and runs in time $O(\poly(V))$ per round, where $V = |\cV(\cP)|$ is the number of vertices of the polytope $\cP$. While for some polytopes (e.g. the simplex $\Delta([N])$) these bounds are reasonable, some common polytopes have an exponentially large number of vertices (even those with a compact representation as the intersection of a small number of half-spaces). Most relevant to us, the Bayesian game polytope $\cP = \Delta([N])^C$ has $V = N^C$ vertices, which is exponential in $C$. This dependence makes this algorithm unusable in settings with even a moderate number of contexts. 

In this section we present two efficient contextual learning algorithms: one that  simply runs a low-swap-regret algorithm independently for each context, and a more complex algorithm which incorporates the idea of swapping ``between'' different contexts. While we do not show these algorithms are low-polytope regret (the first one provably is not), we do show that these algorithms are somewhat robust to strategic behavior by the optimizer, in that we can still upper bound the maximum reward of optimizer as a function of the game $G$.

\subsubsection{Running a low-swap-regret algorithm per context}

We begin by analyzing the simple algorithm that runs an independent low-swap-regret algorithm for each context (Algorithm \ref{alg:alg1}).

\begin{algorithm2e}[h]
  \caption{Running a low-swap-regret algorithm per context.} \label{alg:alg1}
  Learner initializes $C$ copies ($\mathcal{A}_1, \mathcal{A}_2, \dots \mathcal{A}_{C}$) of a low-swap-regret algorithm over $N$ arms and $T$ rounds. \\
  \For{$t \gets 1$ \KwTo $T$}{
    \For{$c \gets 1$ \KwTo $C$}{
        Learner receives mixed strategy $\beta^t_{c} \in \Delta([N])$ from algorithm $\cA_{c}$.\\
    }
    Learner plays $\beta^t(c):[C]\rightarrow\Delta([N])$ given by $\beta^t(c) = \beta^t_{c}$. \\
    Optimizer plays mixed strategy $\alpha^t \in \Delta([M])$.\\
    \For{$c \gets 1$ \KwTo $C$}{
        Learner updates algorithm $\cA_{c}$ with the rewards $u_{L}(\alpha^t, j, c)$ for $j \in [N]$.
    }
  }
\end{algorithm2e}

Since running a low external-regret algorithm per context results in a contextual learning algorithm with low external-regret, one may naturally suspect that running a low internal-regret algorithm per context should result in a contextual learning algorithm with low (polytope) swap-regret. Interestingly, this is not the case.

\begin{theorem}\label{thm:alg1bad}
There exist Bayesian games $G$ where if the optimizer plays $G$ for $T$ rounds against a learner running Algorithm \ref{alg:alg1} (for some choice of swap-regret algorithm), the optimizer can get at least $\Val(G)T + \Omega(T)$ reward.
\end{theorem}

Although Theorem \ref{thm:alg1bad} shows that it is possible for an optimizer to gain significantly (over their Stackelberg value) by strategizing, we can upper bound the extent to which this occurs in terms of a different Stackelberg value-like benchmark. 

Given a Bayesian game $G$, let $G_{1}, G_{2}, \dots, G_{C}$ be the standard games induced by the $C$ different contexts. We define the per-context Stackelberg value $\ConVal(G)$ as the weighted average of the Stackelberg values of the games $G_{c}$:

$$\ConVal(G) = \sum_{c=1}^{C} p_{c}\Val(G_{c}).$$

Alternatively, one can interpret $\ConVal(G)$ as the maximum amount the optimizer can receive by playing a fixed strategy, if the optimizer's strategy too is allowed to depend on contexts. We show that $\ConVal(G)$ upper bounds the optimizer's per-round reward when playing against Algorithm \ref{alg:alg1}.

\begin{theorem}\label{thm:alg1bound}
Let $G$ be a Bayesian game. If an optimizer plays $G$ for $T$ rounds against a learner running Algorithm \ref{alg:alg1}, the optimizer will receive reward at most $\ConVal(G)T + o(T)$.
\end{theorem}
\begin{proof}
Since each sub-algorithm $\cA_c$ of Algorithm \ref{alg:alg1} is low-swap regret, we know (via the results of~\cite{deng2019strategizing}) or Theorem \ref{thm:poly_regret} applied to standard games) that for each context $c$ we must have that:

$$\sum_{t=1}^{T} u_{O}(\alpha^{t}, \beta^{t}(c), c) \leq \Val(G_{c})T + o(T).$$

\noindent
Summing this over $c$ (weighting by $p_{c}$) we have that

$$\sum_{t=1}^{T} u_{O}(\alpha^{t}, \beta^{t}) \leq \ConVal(G)T + o(T).$$
\end{proof}

\subsubsection{A swap-regret reduction for contextual learning}
\label{sec:alg2corr}
We now present a second algorithm (Algorithm \ref{alg:alg2p}) that avoids some of the pitfalls of Algorithm \ref{alg:alg1} (notably, it works on the example of Theorem \ref{thm:alg1bad} and achieves a stronger analogue of Theorem \ref{thm:alg1bound}). Unlike Algorithm \ref{alg:alg1}, which employs a low-swap-regret algorithm as a blackbox, in Algorithm \ref{alg:alg2p} we do something more akin to the reduction of \cite{BlumMansour2007} by running several external regret algorithms in parallel and computing a ``steady state'' distribution.

However, in our setting this notion of ``steady state'' is significantly more involved than in the classical external to internal regret reduction. Whereas the steady state of the reduction of \cite{BlumMansour2007} can be expressed as the stationary distribution of a Markov chain, our ``steady state'' distribution is the fixed point of a system of degree 2 polynomials. This poses several challenges, both for showing this steady state exists and for computing it (indeed, we do not currently have a polynomial-time algorithm for finding this steady state, but we do have several heuristic approaches that appear to work well unlike in the case of the inefficient algorithm of Section \ref{sec:lowpolyalg}). 

\begin{algorithm2e}[h]
  \caption{Swapping between contexts.}
  For each $c \in [C]$ and $j \in [N]$, the learner initializes an instance $\cA_{c,j}$ of a low external-regret algorithm over $N + C$ arms and $T$ rounds. \\
  \For{$t \gets 1$ \KwTo $T$}{
    \For{$c \gets 1$ \KwTo $C$}{
        Learner receives a distribution $\gamma^{t}_{c, j} \in \Delta([N + C])$ from algorithm $\cA_{c, j}$.\\
    }
    Learner plays a $\beta^{t}(c):[C] \rightarrow\Delta([N])$ satisfying
    
    \begin{equation}\label{eq:markov}
    \beta^{t}(c)_{j} = \sum_{j'=1}^{N}\beta^{t}(c)_{j'}\left( \gamma^{t}_{c, j', j} + \sum_{c'=1}^{C}\gamma^{t}_{c, j', N+c'}\beta^{t}(c')_{j}\right).
    \end{equation}
    
    Optimizer plays mixed strategy $\alpha^t \in \Delta([M])$.\\
    
    \For{$c \gets 1$ \KwTo $C$}{
        Learner updates algorithm $\cA_{c, j}$ with the reward $r^{t}_{c, j} \in [-1, 1]^{N+C}$, where for $1 \leq j' \leq N$,
        
        $$r^{t}_{c, j, j'} = \beta^{t}(c)_{j} \cdot u_{L}(\alpha, j', c),$$
        
        and where for $1 \leq c' \leq C$,
        
        $$r^{t}_{c, j, N+c'} = \beta^{t}(c)_{j} \cdot u_{L}(\alpha, \beta^{t}(c'), c).$$
    }
  }
\label{alg:alg2p}
\end{algorithm2e}

We attempt to give some motivation behind the workings of Algorithm \ref{alg:alg2p} and where this polynomial system arises from. Intuitively, what we would like to do is run $C$ copies of a low-swap-regret algorithm, one for responsible for learning a good distribution over actions for each context. Each copy runs over $N+C$ arms. $N$ of these arms correspond to pure actions for the learner. Where this gets tricky is the other $C$ actions correspond to other contexts -- specifically, they correspond to the distributions over arms output by the \textit{original $C$ low-swap-regret algorithms}.

Such an algorithm has the property that we do not only get low-swap-regret between actions in the same context (as with Algorithm \ref{alg:alg1}), but also low-swap-regret ``between'' different contexts (in particular, this will allow us to avoid the bad example in Theorem \ref{thm:alg1bad}). The problem is that this algorithm as described is inherently circular; the output of these low regret algorithms are also arms. We resolve this by constructing an appropriate fixed point problem. 

Before we analyze the guarantees of our algorithm, we show that it is well-defined. In particular, we show that we can always find a solution to the system of equations \eqref{eq:markov} defining $\beta^{t}(c)$ in Algorithm \ref{alg:alg2p}.

\begin{lemma}\label{lem:solvable}
For any choice of $\gamma_{c, j} \in \Delta([N+C])$, there always exists a $\beta:[C] \rightarrow \Delta([N])$ that satisfies

$$\beta(c)_{j} = \sum_{j'=1}^{N}\beta(c)_{j'}\left( \gamma_{c, j', j} + \sum_{c'=1}^{C}\gamma_{c, j', N+c'}\beta(c')_{j}\right).$$
\end{lemma}

Lemma \ref{lem:solvable} shows that a solution $\beta^t$ always exists to the above system of polynomial equations, but does not describe how to find it efficiently. Indeed, this is a problem -- unlike in the linear case, solving systems of quadratic equations over multiple variables can be NP-hard.

Nonetheless, specific quadratic programs and systems can be solved efficiently, and we conjecture that it is possible to approximately solve this system in polynomial time. Specifically, taking inspiration from the power method for computing the stationary distribution for Markov chains, we suspect that starting from a uniform $\beta$ and iterating the map

\begin{equation}\label{eq:quad_iter}
\beta(c)_{j} \rightarrow \sum_{j'=1}^{N}\beta(c)_{j'}\left( \gamma_{c, j', j} + \sum_{c'=1}^{C}\gamma_{c, j', N+c'}\beta(c')_{j}\right)
\end{equation}

\noindent
quickly converges to the fixed point shown to exist in Lemma \ref{lem:solvable} (this is supported by some computer simulations). We pose this as an open question.

\begin{question}
Let $\beta^{(\tau)}$ be the element of $\Delta([N])^{C}$ obtained by starting with the $\beta^{(0)}$ that maps each $c \in [C]$ to the uniform distribution over $[N]$ and iterating the map \eqref{eq:quad_iter} $\tau$ times. Then if $\beta$ is the fixed point in Lemma \ref{lem:solvable}, is it the case that

$$||\beta - \beta^{\tau}||_{\infty} \leq \eps$$

\noindent
for some value of $\tau = \poly(N, C, \log 1/\eps)$?
\end{question}

As with Algorithm \ref{alg:alg1}, we can upper bound the extent to which an optimizer can gain by strategizing against Algorithm \ref{alg:alg2p}. To do so, we
first need to define a version of correlated equilibria in Bayesian games.

A \textit{correlated equilibrium} of a Bayesian game $G$ is specified by a set of $C$ distributions $\cF_{c}$ over $[M] \times [N]$, where the strategy profile $(i, j) \in [M] \times [N]$ occurs with probability $p_{i, j}(c)$ in $\cF_{c}$. Now, imagine a variant of the game $G$ (analogous to how correlated equilibria in standard games are defined) where the learner begins by communicating the context $c$ to a third-party ``correlator''. The correlator then draws a strategy profile $(i, j)$ from $\cF_{c}$ and tells the optimizer to play $i$ and the learner to play $j$. In order for this set of distributions to be a correlated equilibrium, they must satisfy the following properties\footnote{Note that we only include constraints for the learner, since we later choose the equilibrium which is most favorable to the optimizer. This one-sidedness causes this definition to be slightly different than the traditional definition for correlated equilibria.}:

\begin{itemize}
    \item The learner must have no incentive to misreport their type. That is, for all $c, c' \in [C]$, we must have:
    
    $$\E_{(i, j) \sim \cF_{c}}\left[u_{L}(i, j, c)\right] \geq \E_{(i, j) \sim \cF_{c'}}\left[u_{L}(i, j, c)\right].$$
    
    \item The learner must have no incentive to ``swap'' their action. That is, for every swap rule $\pi: [N] \rightarrow [N]$, we must have:
    
    $$\E_{(i, j) \sim \cF_{c}}\left[u_{L}(i, j, c)\right] \geq \E_{(i, j) \sim \cF_{c}}\left[u_{L}(i, \pi(j), c)\right].$$
\end{itemize}

We define the \textit{correlated Stackelberg value} $\CorrVal(G)$ of $G$ to be the maximum expected value for the optimizer over all correlated equilibria of $G$; i.e.,

$$\CorrVal(G) = \max \E_{c \sim \cD}\E_{(i, j) \sim \cF_{c}}\left[u_{O}(i, j, c)\right],$$

\noindent
where the maximum is over all collections $\{\cF_{c}\}$ are correlated equilibria of $G$. It turns out that for all Bayesian games $G$, $\CorrVal(G) \leq \ConVal(G)$; in particular, this follows from the known fact that the Stackelberg value of a standard game is equal to the maximum utility of the Stackelberg player in a correlated equilibrium of the game (see \cite{conitzer2016stackelberg} or \cite{von2010leadership} for a proof). In particular, if we remove the constraint that the learner has no incentive to misreport their type from the above criteria, we recover an alternate definition for $\ConVal(G)$.

We now show that $\CorrVal(G)$ upper bounds the optimizer's per-round reward when playing against Algorithm \ref{alg:alg2p} (thus providing a stronger bound than Theorem \ref{thm:alg1bound}). 

\begin{theorem}\label{thm:alg2pbound}
Let $G$ be a Bayesian game. If an optimizer plays $G$ for $T$ rounds against a learner running Algorithm \ref{alg:alg2p}, the optimizer will receive reward at most $\CorrVal(G)T + o(T)$.
\end{theorem}

Unlike Algorithm \ref{alg:alg1}, we have no example of a Bayesian game $G$ where the optimizer can do better than $\Val(G)T$ when playing against Algorithm \ref{alg:alg2p}. It is an interesting open question whether Algorithm \ref{alg:alg2p} is low polytope swap regret (or otherwise prevents the optimizer from beating their Stackelberg value).

\begin{question}
Is there a game $G$ where the optimizer can get $\Val(G)T + \Omega(T)$ reward when playing against a learner running Algorithm \ref{alg:alg2p}? Is Algorithm \ref{alg:alg2p} a low polytope regret contextual learning algorithm?
\end{question}

We suspect the answer to the above answer is no: in particular, we suspect that the exponential regret and complexity of the generic algorithm in Section \ref{sec:lowpolyalg} is necessary, and that there is no contextual learning algorithm that achieves polytope swap regret $O(\poly(C, K)\sqrt{T})$ and no efficient (running in $\poly(C, K)$ time per iteration) contextual learning algorithm that achieves polytope swap regret $o(T)$. 

\subsection{Computing Bayesian Stackelberg equilibria is hard}\label{sec:hard_equilibrium}

One reason to suspect that minimizing polytope swap regret is hard for a contextual learner is that, surprisingly, the optimizer faces a provably hard optimization problem when playing against a low polytope swap regret learner in a Bayesian game $G$. Specifically, we show (via closely following a proof of a similar result for typed principal-agent problems in \cite{guruganesh2021contracts}) that it is APX-hard to compute the Stackelberg value $\Val(G)$ for a Bayesian game (and thus hard to compute the corresponding Stackelberg strategy). 

\begin{theorem}\label{thm:apx-hardness}
It is APX-hard to compute the Stackelberg value for the optimizer in a Bayesian game. That is, there exists a constant $\eps > 0$ such that given a Bayesian game $G$ and a value $\Val' > 0$ it is NP-hard to distinguish between the cases $\Val(G) \leq (1-\eps)\Val'$ and $\Val(G) \geq \Val'$. 
\end{theorem}

\section{Achieving the Stackelberg value in Bayesian games and polytope games}\label{app:achieve_stack}

We show (as mentioned in Section \ref{sec:intro_equilibria}) that an optimizer can always achieve the Stackelberg value in Bayesian game $G$, under mild conditions on $G$ (essentially, every pure strategy for the learner is a strict best response for some strategy for the optimizer). In fact, we will show this for an arbitrary polytope game $G$, from which the conclusion for Bayesian games will immediately follow. Our proof largely follows the analogous proof in \cite{deng2019strategizing} for standard games.

Let $\cP$ be a polytope and let $G$ be a $\cP$-game. We say the game $G$ is non-degenerate if, for each vertex $v$, there exists a strategy for the optimizer $\alpha_v$ where the learner's strict best response is $v$ (i.e., $\BR(\alpha_{v}) = \{v\})$.

\begin{lemma}\label{lem:stack_achieve}
Let $G$ be a non-degenerate game. Then if the optimizer plays $T$ rounds of $G$ against a learner running a low (external) regret algorithm $\cA$, the optimizer can guarantee a reward of at least $\Val(G)T - o(T)$.
\end{lemma}
\begin{proof}
Let $\alpha = (r, s) \in \cQ$ be the optimizer's strategy and let $v \in \cV(\cP)$ be the learner's best-response in the Stackelberg equilibrium of $G$. Because $G$ is non-degenerate, there exists a strategy $\alpha_{v} = (r_{v}, s_{v}) \in \cQ$ for the optimizer where $v$ is the strict best response. Let us define $\delta = \min_{v' \neq v} \langle r_{v}, v\rangle - \langle r_{v}, v'\rangle$; intuitively, $\delta$ represents the margin by which $v$ is a strict best-response.

Assume the algorithm $\cA$ has the guarantee that it incurs at most $R(T) = o(T)$ external regret on any $\cP$-learning instance for $T$ rounds. Set $\eps = \sqrt{R(T)/T}$, and consider what happens when the optimizer plays $\alpha' = (1-\eps)\alpha + \eps\alpha_{v}$ every round for $T$ rounds against $\cA$.

Let $x^{t} \in \cP$ (for $1 \leq t \leq T$) be the response of the learner in round $t$. Let $\rho^{t} \in \Delta(\cV(\cP))$ be an arbitrary vertex decomposition of $x^{t}$. Note that the total utility of the learner can be written as

$$\sum_{v' \in \cV(\cP)} \sum_{t=1}^{T} \rho_{v'}^{t} \langle r', v'\rangle.$$

\noindent
Since the learner has low external regret, we must have that

$$\sum_{v' \in \cV(\cP)} \sum_{t=1}^{T} \rho_{v'}^{t} \langle r', v' - v\rangle \leq R(T).$$

\noindent
Now, note that if $v' \neq v$, our guarantee on $\alpha_{v}$ implies that $ \langle r', v' - v\rangle \geq \eps \delta$. It therefore follows that

$$\sum_{v' \neq v} \sum_{t=1}^{T} \rho_{v'}^{t} = \sum_{t=1}^{T} (1- \rho_{v}^{t}) \leq \frac{R(T)}{\eps\delta} = \frac{\sqrt{R(T)T}}{\delta}.$$

On the other hand, the total utility of the optimizer can be written as 

$$\sum_{v' \in \cV(\cP)} \sum_{t=1}^{T} \rho_{v'}^{t} \langle s', v'\rangle \geq \sum_{t=1}^{T} \rho_{v}^{t} \langle s', v\rangle.$$

Since $\alpha$ and $v$ form a Stackelberg equilibrium for $G$, $\langle s', v\rangle \geq (1-\eps)\langle s, v\rangle = (1-\eps)\Val(G)$. It follows that the optimizer's utility is at least

$$\sum_{t=1}^{T} \rho_{v}^{t} \langle s', v\rangle \geq \left(\sum_{t=1}^{T} \rho_{v}^{t}\right)(1-\eps)\Val(G) \geq T\left(1 - \frac{\eps}{\delta}\right)(1-\eps)\Val(G) \geq \Val(G)T - o(T).$$

\end{proof}

\section{Omitted proofs} \label{app:omitted}

\subsection{Proof of Theorem \ref{thm:main_standard}}
\begin{proof}[Proof of Theorem \ref{thm:main_standard}]
Since $\cA$ is not a low-swap-regret algorithm, there exists some $\gamma > 0$ and positive integer $N$ such that for infinitely many values of $T$, there exists an online learning instance with $N$ actions and $T$ rounds where $\cA$ incurs at least $\gamma T$ swap regret.

If we let the game $G$ depend on $T$, then this theorem directly follows from Lemma \ref{lem:standard}. But also note that for a fixed value of $N$, the ultimate game $G$ constructed in Lemma \ref{lem:standard} depends solely on the optimal swap function $\pi$. Since there are infinitely many values of $T$ for which there exists a bad online-learning instance but only finitely many ($N^N$) different swap functions, infinitely many of these values of $T$ will have the same swap function $\pi$ and hence the same game $G$. This game satisfies the theorem statement.
\end{proof}

\subsection{Proof of Lemma \ref{lem:linear}}
\begin{proof}[Proof of Lemma \ref{lem:linear}]
Let $r^{t} \in [-1, 1]^d$ be the rewards of the bad instance of the $\cP$-learning problem, and let $x^{t} \in \cP$ be the corresponding actions played by $\cA$. Since the linear swap regret on this instance equals $R$, we know that there exists an $M \in \cM(\cP)$ such that

\begin{equation}\label{eq:hilinsr}
\sum_{t=1}^{T} \langle r^{t}, Mx^{t} \rangle  - \sum_{t=1}^{T} \langle r^{t}, x^{t} \rangle = R.
\end{equation}

\noindent
We rewrite this as

\begin{equation}\label{eq:hilinsr2}
\sum_{t=1}^{T} \langle r^{t}, (M-I)x^{t} \rangle = R.
\end{equation}

To define the $\cP$-game $G$ we simply need to specify the polytope $\cQ$. We will define $\cQ$ to be the polytope

$$\cQ = \left\{\left(y, \frac{1}{\lambda+1}(M-I)^{\intercal}y\right) \middle\mid\, y \in [-1, 1]^d \right\}.$$

\noindent
Note that scaling by $1/(\lambda+1)$ guarantees that $\cQ \subseteq [-1, 1]^{2d}$. In particular, since $||M||_1 \leq \lambda$, $||M^{\intercal}||_{\infty} \leq \lambda$, so $||(M-I)^{\intercal}||_{\infty} \leq \lambda + 1$, and therefore $(M-I)^{\intercal}y \in [-(\lambda+1), (\lambda+1)]^d$.

Now, consider an optimizer who in round $t$ plays the action $q^{t} = (r^{t}, (M-I)^{\intercal}r^{t})$. The learner, running $\cA$, will see exactly the sequence of rewards $r^{t}$ and therefore on each round $t$ play action $x^{t} \in \cP$. The total reward of the optimizer is then given by

$$\sum_{t=1}^{T} \langle (M-I)^{\intercal}r^{t}, x^{t} \rangle = \sum_{t=1}^{T}\langle r^{t}, (M-I)x^{t} \rangle = R.$$

On the other hand, we claim the Stackelberg value $\Val(G)$ of the game is at most zero. To see this, note that if the optimizer plays $q = (r, (M-I)^{\intercal}r) \in \cQ$ and the learner best responds by playing $x \in \cP$, then the optimizer's payoff is $\langle (M-I)^{\intercal}r, x \rangle = \langle r, (M-I)x \rangle = \langle r, Mx\rangle - \langle r, x\rangle$. But since $Mx \in \cP$ and $x$ is the learner's best response to $q$, we must have that $\langle r, Mx \rangle \leq \langle r, x\rangle$. It follows that $\Val(G) \leq 0$.
\end{proof}

\subsection{Proof of Theorem~\ref{thm:linear}}

\begin{proof}[Proof of Theorem~\ref{thm:linear}]
If $\cA$ does not have low linear swap regret, then there exists a $\gamma > 0$ such that for infinitely many values of $T$, there exists a $\cP$-learning instance where $\cA$ incurs at least $\gamma T$ linear swap regret.

As with standard games, if the game $G$ could depend on $T$, then we would be done by Lemma \ref{lem:linear}. We argue that for a fixed polytope $\cP$, the procedure in Lemma \ref{lem:linear} can generate only finitely many different games. In particular, note that for a fixed problem instance, the game $G$ is defined entirely by a matrix $M$ in $\cM(\cP)$. Now, while there may be infinitely many matrices in $\cM(\cP)$, we will show that $\cM(\cP)$ is actually a convex polytope in $\mathbb{R}^{d^2}$, and we can always choose $M$ to be a vertex of $\cM(\cP)$ (of which there are finitely many).

To see that $\cM(\cP)$ is a polytope, note that in order to guarantee the constraint that $x \in \cP$ implies $Mx \in \cP$ it is necessary and sufficient that $M$ map each vertex of $\cP$ to a point within $\cP$. Therefore for each vertex $v$ (of the finitely many vertices) of $\cP$ and each halfspace $Ax \leq b$ (of the finitely many halfspaces) defining $\cP$, we obtain the linear constraint $AMv \leq b$ on the matrix $M$. These linear constraints define $\cM(\cP)$.

Now, note that for any fixed rewards $r^{t}$ and actions $x^{t}$, the sum $\sum_{t=1}^{T} \langle r^{t}, Mx^{t}\rangle$ is linear in the matrix $M$. This means that over all matrices $M \in \cM(\cP)$, it is maximized at one of the vertices of $\cM(\cP)$. In particular, in Lemma \ref{lem:linear}, we can always ensure that the $M$ we choose to define $G$ is one of these finitely many vertices (and therefore there are only finitely many possible $\cP$-games we can generate). One of these games $G$ must occur for infinitely many values of $T$, and that game satisfies the theorem statement.
\end{proof}

\subsection{Proof of Theorem \ref{thm:poly_regret}}

\begin{proof}[Proof of Theorem \ref{thm:poly_regret}]
Consider the transcript of this game when played for $T$ rounds. Let $q^{t} = (r^{t}, s^{t}) \in \cQ$ be the optimizer's action in round $t$, and let $x^{t} \in \cP$ be the learner's action in round $t$. Since $\cA$ is low swap regret, we know there exist vertex decompositions $\rho^{t}$ of $x^{t}$ such that for any swap function $\pi : \cV(\cP) \rightarrow \cV(\cP)$,

\begin{equation} \label{eq:poly1}
\sum_{t=1}^{T} \left\langle r^{t}, \overline{\pi(\rho^{t})} \right\rangle - \sum_{t=1}^{T} \langle r^{t}, x^{t} \rangle = o(T).
\end{equation}

\noindent
Using the fact that $x^{t} = \overline{\rho^{t}}$, we can rewrite \eqref{eq:poly1} as

\begin{equation} \label{eq:poly2}
\sum_{t=1}^{T} \left\langle r^{t}, \overline{\pi(\rho^{t})} - \overline{\rho^{t}} \right\rangle = o(T).
\end{equation}

\noindent
Decomposing \eqref{eq:poly2} over vertices in $\cV(\cP)$, this becomes

\begin{equation} \label{eq:poly3}
\sum_{v \in \cV(\cP)} \sum_{t=1}^{T} \rho^{t}_{v}\left\langle r^{t}, \pi(v) - v \right\rangle = o(T).
\end{equation}

Now, let $\sigma_{v} = \sum_{t=1}^{T} \rho^{t}_{v}$, let $\tilde{r}_{v} = (\sum_{t=1}^{T} \rho^{t}_{v}r^{t})/\sigma_{v}$, and let $\tilde{s}_{v} = (\sum_{t=1}^{T}\rho^{t}_{v}s^{t})/\sigma_{v}$. Note that $(\tilde{r}_{v}, \tilde{s}_{v})$ is a convex combination of the optimizer's actions $(r^{t}, s^{t})$ and therefore belongs to $\cQ$. 

We can once again rewrite \eqref{eq:poly3} as

\begin{equation} \label{eq:poly4}
\sum_{v \in \cV(\cP)} \sigma_{v} \langle \tilde{r}_{v}, \pi(v) - v \rangle = o(T).
\end{equation}

Now, note that we can write the optimizer's utility in the form $\sum_{v \in \cV(\cP)} \sigma_{v} \langle \tilde{s}_{v}, v \rangle$. Assume that the statement of the theorem is not true, namely that for infinitely many $T$, we have that

\begin{equation}\label{eq:opt1}
    \sum_{v \in \cV(\cP)} \sigma_{v} \langle \tilde{s}_{v}, v \rangle \geq (\Val(G) + \gamma)T
\end{equation}

\noindent
for some $\gamma > 0$. Let $\BR(\tilde{s}_{v}) \in \cV(\cP)$ be the learner's best response to the optimizer's action $(\tilde{r}_{v}, \tilde{s}_{v})$ in $G$. By the definition of the Stackelberg value of $G$, we have that 

\begin{equation}\label{eq:opt2}
\sum_{v \in \cV(\cP)} \sigma_{v} \langle \tilde{s}_{v}, \BR(\tilde{s}_{v}) \rangle \leq \Val(G)T
\end{equation}

\noindent
Subtracting \eqref{eq:opt2} from \eqref{eq:opt1}, we obtain

\begin{equation}\label{eq:opt3}
\sum_{v \in \cV(\cP)} \sigma_{v} \langle \tilde{s}_{v}, v - \BR(\tilde{s}_{v}) \rangle \geq \gamma T
\end{equation}

\noindent
But this contradicts \eqref{eq:poly4} for the swap function $\pi(v) = \BR(\tilde{s}_v)$. The theorem follows.
\end{proof}

\subsection{Proof of Theorem \ref{thm:poly_lin_sep}}

\begin{proof}[Proof of Theorem \ref{thm:poly_lin_sep}]
Note that when $K=C=2$, there are four vertices in $\cV(\cP)$: $v_{11} = (1, 0; 1, 0)$, $v_{12} = (1, 0; 0, 1)$, $v_{21} = (0, 1; 1, 0)$ and $v_{22} = (0, 1; 0, 1)$. The learner will play $v_{11}$ for the first $T/4$ rounds, $v_{12}$ for the second $T/4$ rounds, $v_{21}$ for the third $T/4$ rounds, and $v_{22}$ for the last $T/4$ rounds. Simultaneously, the rewards will be $r_{11} = v_{11}$ for the first $T/4$ rounds, $r_{12} = v_{12}$ for the next $T/4$ rounds, $r_{21} = v_{21}$ for the next $T/4$ rounds, and finally, $r_{22} = v_{11}$ for the last $T/4$ rounds. \textbf{Note that for the last $T/4$ rounds, $r_{22} \neq v_{22}$, but instead $r_{22} = v_{11}$.}

Let us begin by considering the polytope swap regret of this instance. Note that each action of the learner is already a vertex belonging to $V(\cP)$. This means that we have no freedom in choosing the vertex partition; each vertex partition must put all of its weight on the action itself. Then, the swap function which maps $(v_{11}, v_{12}, v_{21}, v_{22}) \rightarrow (v_{11}, v_{12}, v_{21}, v_{11})$ increases the learner's utility by $2$ points per round for the last $T/4$ rounds, and therefore $\PolySwapReg \geq T/2$.

On the other hand, there is no linear function which maps vertices $v_{11}, v_{12}, v_{21}$ to themselves but which maps the vertex $v_{22}$ to the vertex $v_{11}$. To upper bound the linear swap regret, we can find the best linear swap function by checking each of the extremal points of $\cM(\cP)$. It turns out there are $64$ such extremal swap functions, and a straightforward computation shows that none perform better than the identity function. It follows that $\LinSwapReg = 0$. 
\end{proof}

\subsection{Proof of Theorem~\ref{thm:poly_lin_sep_bayesian}}

\begin{proof}[Proof of Theorem~\ref{thm:poly_lin_sep_bayesian}]
We extend the example separating linear swap regret and polytope swap regret in Theorem \ref{thm:poly_lin_sep}. We generalize this to a game $G$ where the optimizer has four actions, each corresponding to one of the four quarters of the game (so the optimizer plays action 1 for the first $T/4$ rounds, action 2 for the next $T/4$ rounds, etc.). For each action $i$, the optimizer has a reward vector $s_i \in \mathbb{R}^{4}$, denoting that if the learner plays a mixed action $\alpha \in \cP$, then the optimizer gets reward $\langle \alpha, s_i\rangle$. (In the language of polytope games, the action set for the optimizer is $\cQ = \conv(\{(v_{i}, s_{i})\}_{i=1}^{4})$. If we set:

\begin{eqnarray*}
s_1 &=& (0.00, 0.26; 0.60, 0.21)\\
s_2 &=& (0.05, 0.17; 0.45, 0.68)\\
s_3 &=& (0.16, 0.25; 0.33, 0.20)\\
s_4 &=& (0.16, 0.68; 0.22, 0.44),
\end{eqnarray*}

\noindent
then we can check that while $\Val(G) = 0.74$, the optimizer gets $0.7575T$ reward from playing the mentioned trajectory of actions (action 1, then 2, then 3, then 4). 
\end{proof}

\subsection{Proof of Theorem \ref{thm:alg1bad}}

\begin{proof}[Proof of Theorem \ref{thm:alg1bad}]
We adapt an example of \cite{braverman2018selling}. We describe a Bayesian game $G$ with $M = N = C = 2$. We will interpret this game as a game where the optimizer is trying to sell an item to a learner whose value is specified by the context $c$. The optimizer can set one of two prices for the item ($0$ or $1$), the learner's value for the item is either $1/4$ or $1/2$ (depending on the value of $C$), and the learner must choose whether to buy or not buy the item (without seeing the price). Formally, we have

\begin{eqnarray*}
u_{L}(i, j, c) &=& \left(\frac{c}{4} - i\right)j \\
u_{O}(i, 1, c) &=& ij,
\end{eqnarray*}

\noindent
where $i \in \{0, 1\}, j\in \{0, 1\}, c \in [2]$, and let $\cD$ be the uniform distribution over the two contexts. It is straightforward to verify that $\Val(G) = 1/4$, and this is achievable if the optimizer plays $\alpha = (3/4, 1/4)$ (i.e., the optimizer sets a price of $1/4$ for the item).

We now show that an optimizer can get $T/4 + \Omega(T)$ when playing against a learner running Algorithm \ref{alg:alg1} for $T$ steps. As in the description of the algorithm, let $\beta^{t}: [C] \rightarrow \Delta([N])$ be the strategy the learner plays in round $t$. 

Note that since $N=2$, low-swap-regret over $N$ actions is equivalent to low-regret over $N$ actions, so we can just assume that both of the sub-algorithms $\cA$ are a low-regret algorithm such as Hedge. The only property of the learning algorithm that we will need is the following: assume $R^{t}(c)_0$ is the total reward of action $0$ up until round $t$ in context $c$ and that $R^{t}(c)_1$ is the total reward of action $1$ up until round $t$ in context $c$. Then there is a sublinear function $r(T) = o(T)$ such that if $R^{t}(c)_1 - R^{t}(c)_0 > r(t)$, then $\beta^{t}(c)_{0} < o(T)$.

Now consider the following strategy for the optimizer: the optimizer will play $0$ for the first $T/2$ rounds and $1$ for the last $T/2$ rounds. Let us consider the values of $R^{t}(c)_0$ and $R^{t}(c)_1$. When $j=0$, the learner does not buy the item and $u_{L} = u_{O} = 0$, so $R^{t}(c)_{0} = 0$. On the other hand $R^{t}(c)_1 = ct/4$ for $t \in [1, T/2]$ and $R^{t}(c)_1 = ct/4 - (t - T/2)$ for $t \in [1, T/4]$. Note that when $c = 2$, $R^{t}(c) \geq 0$ for all $t$ (and in fact $R^{t}(c) \geq o(T)$ for almost all $t$), so $\beta^{t}(c)_{1} \geq 1-o(T)$ for almost all $t$. It follows that the optimizer receives utility $(1/2) \cdot (T/2) \cdot 1 - o(T) = T/4 - o(T)$ from the context $c=2$.

On the other hand, when $c = 1$, $R^{t}(c) \geq 0$ for all $1 \leq t \leq 3T/4$. A similar argument shows that in this case, the optimizer receives utility $(1/2) \cdot (T/4) = T/8 - o(T)$ from this context. Overall, the optimizer receives utility $3T/8 - o(T) = T/4 + \Omega(T)$. 
\end{proof}

\subsection{Proof of Lemma \ref{lem:solvable}}

\begin{proof}[Proof of Lemma \ref{lem:solvable}]
Fix a $\beta: [C] \rightarrow \Delta([N])$, and consider the function $\beta': [C] \rightarrow \mathbb{R}^n$ defined via:

$$\beta'(c)_{j} = \sum_{j'=1}^{N}\beta(c)_{j'}\left( \gamma_{c, j', j} + \sum_{c'=1}^{C}\gamma_{c, j, N+c'}\beta(c')_{j}\right).$$

We will show that $\beta'(c)$ is also an element of $[C] \rightarrow \Delta([N])$. Note that since both $\beta(c)$ and $\gamma_{c, j}$ are distributions, from the above equation we can observe that $\beta'(c)_{j} \in [0, 1]$ for all $j$. It thus suffices to check that $\sum_{j}\beta'(c)_{j} = 1$ for each $c \in [C]$. We perform this computation:

\begin{eqnarray*}
\sum_{j=1}^{N}\beta'(c)_{j} &=& \sum_{j=1}^{N}\sum_{j'=1}^{N}\beta(c)_{j'}\left( \gamma_{c, j', j} + \sum_{c'=1}^{C}\gamma_{c, j', N+c'}\beta(c')_{j}\right) \\
&=& \sum_{j=1}^{N}\sum_{j'=1}^{N}\beta(c)_{j'} \gamma_{c, j', j} + \sum_{j=1}^{N}\sum_{j'=1}^{N}\sum_{c'=1}^{C}\beta(c)_{j'}\gamma_{c, j', N+c'}\beta(c')_{j} \\
&=& \sum_{j=1}^{N}\sum_{j'=1}^{N}\beta(c)_{j'} \gamma_{c, j', j} + \sum_{j'=1}^{N}\sum_{c'=1}^{C}\beta(c)_{j'}\gamma_{c, j', N+c'}\sum_{j=1}^{N}\beta(c')_{j}\\
&=& \sum_{j=1}^{N}\sum_{j'=1}^{N}\beta(c)_{j'} \gamma_{c, j', j} + \sum_{j'=1}^{N}\sum_{c'=1}^{C}\beta(c)_{j'}\gamma_{c, j', N+c'}\\
&=& \sum_{j'=1}^{N}\beta(c)_{j'} \sum_{k=1}^{N+C}\gamma_{c, j', k}\\
&=& \sum_{j'=1}^{N}\beta(c)_{j'} = 1.
\end{eqnarray*}

Now, since this mapping from $\beta$ to $\beta'$ is a continuous mapping from $\Delta([N])^{C}$ to $\Delta([N])^C$, by Brouwer's fixed point theorem, there must exist an element $\beta \in \Delta([N])^C$ which is fixed by this mapping (and therefore satisfies the equation in the theorem statement).
\end{proof}

\subsection{Proof of Theorem \ref{thm:alg2pbound}}

\begin{proof}[Proof of Theorem \ref{thm:alg2pbound}]
We will show that the average strategy profile of the optimizer and learner (over all $T$ rounds) is approximately a correlated equilibrium of $G$, from which the conclusion will follow.

We begin by showing that in this average strategy profile, the learner cannot benefit much by ``misreporting'' their type, i.e.

\begin{equation}\label{eq:des1}
\sum_{t=1}^{T} u_{L}(\alpha^{t}, \beta^{t}(c), c) \geq \sum_{t=1}^{T} u_{L}(\alpha^{t}, \beta^{t}(c'), c) - o(T).\end{equation}

To show this, note that the external regret guarantees of algorithm $\cA_{c, j}$ ensure that for any $k \in [N+C]$,

\begin{equation}\label{eq:lowreg}
\sum_{t=1}^{T} \langle r_{c, j}^{t}, \gamma_{c, j}^{t} \rangle \geq \sum_{t=1}^{T} r^{t}_{c, j, k} - o(T).
\end{equation}

\noindent
Let us pick $k = N + c'$. Then the above inequality becomes 

\begin{equation}\label{eq:ineq1}
    \sum_{t=1}^{T} \langle r_{c, j}^{t}, \gamma_{c, j}^{t} \rangle \geq \sum_{t=1}^{T} \beta^{t}(c)_{j} \cdot u_{L}(\alpha^{t}, \beta^{t}(c'), c) - o(T).
\end{equation}

On the other hand, note that

\begin{equation} 
\langle r_{c, j}^{t}, \gamma_{c, j}^{t} \rangle = \beta^{t}(c)_j \left(\sum_{j'=1}^{N}\gamma_{c, j, j'} u_{L}(\alpha, j', c) + \sum_{c'=1}^{C}\gamma_{c, j, N+c'} u_{L}(\alpha, \beta^{t}(c'), c) \right) = \beta^{t}(c)_{j} u_{L}(\alpha, \beta^{t}(c), c),
\end{equation}

\noindent
where the last equality follows as a consequence of \eqref{eq:markov}. Substituting this into \eqref{eq:ineq1}, we have that:

\begin{equation}\label{eq:ineq2}
    \sum_{t=1}^{T} \beta^{t}(c)_{j} \cdot u_{L}(\alpha^{t}, \beta^{t}(c), c) \geq \sum_{t=1}^{T} \beta^{t}(c)_{j} \cdot u_{L}(\alpha^{t}, \beta^{t}(c'), c) - o(T)
\end{equation}

Summing \eqref{eq:ineq2} over all $j \in [N]$ (and using the fact that $\sum_{j}\beta^{t}(c)_{j} = 1$), we obtain our desired inequality \eqref{eq:des1}.

We now show that in the average strategy profile, the learner cannot benefit much by applying a swap function to their action. Specifically, let $\pi: [N] \rightarrow [N]$ be an arbitrary swap function. We will show that

\begin{equation}\label{eq:des2}
\sum_{t=1}^{T} u_{L}(\alpha^{t}, \beta^{t}(c), c) \geq \sum_{t=1}^{T} u_{L}(\alpha^{t}, \pi(\beta^{t}(c)), c) - o(T).\end{equation}

Here for $\beta \in \Delta([N])$, we write $\pi(\beta)$ to denote the element of $\Delta([N])$ which satisfies $\pi(\beta)_{j} = \sum_{j' \in \pi^{-1}(j)} \beta_{j'}$ (i.e., we can sample from $\pi(\beta)$ by first sampling $j'$ from $\beta$ and then playing $\pi(j')$).

As before, we will start from the external regret guarantee \eqref{eq:lowreg} of the individual algorithm $\cA_{c, j}$. If we fix $k = \pi(j)$ this time, and apply the same logic as before, we obtain the inequality

\begin{equation}\label{eq:ineq3}
    \sum_{t=1}^{T} \beta^{t}(c)_{j} \cdot u_{L}(\alpha^{t}, \beta^{t}(c), c) \geq \sum_{t=1}^{T} \beta^{t}(c)_{j} \cdot u_{L}(\alpha^{t}, \pi(j), c) - o(T).
\end{equation}

\noindent
But now, note that

\begin{equation}
    \sum_{j=1}^{N} \beta^{t}(c)_{j} \cdot u_{L}(\alpha^{t}, \pi(j), c) = u_{L}(\alpha^{t}, \pi(\beta^{t}(c)), c).
\end{equation}

It follows that by summing \eqref{eq:ineq3} over all $j \in [N]$ that we obtain our desired inequality. 

Now, consider the set of $C$ distributions $\cF'_{c}$ over $[M] \times [N]$ given by

$$\Pr_{\cF'_{c}}(i, j) = \frac{1}{T}\sum_{i=1}^{M}\sum_{j=1}^{N} \alpha^{t}_{i}\beta^{t}(c)_{j}.$$

In words, the collection of distributions $\cF'_{c}$ record the average strategy profile played by the optimizer and learner over the $T$ rounds. Now, \eqref{eq:des1} and \eqref{eq:des2} imply that for each $T$, there exists a function $\eps(T) = o(1)$ such that $\cF'_{c}$ form an $\eps$-approximate correlated equilibrium of $G$ (where an $\eps$-approximate equilibrium satisfies the inequalities in the definition of a correlated equilibrium up to a slack of $\eps$). Note that if we let $\CorrVal(G, \eps)$ be the maximum value for the optimizer over all $\eps$-correlated equilibria, then the optimizer will receive a reward of at most $\CorrVal(G, \eps(T))T$ against this optimizer.

But now, note that $\CorrVal(G, \eps)$ converges to $\CorrVal(G)$ as $\eps \rightarrow 0$. In particular, this means that $\CorrVal(G, \eps(T)) - \CorrVal(G) = o(1)$, and therefore the optimizer will receive a reward of at most $\CorrVal(G)T + o(T)$. 
\end{proof}

\subsection{Proof of Theorem~\ref{thm:apx-hardness}}
\begin{proof}[Proof of Theorem~\ref{thm:apx-hardness}]
We adapt a proof of \cite{guruganesh2021contracts} which shows hardness of computing optimal contracts in principal agent problems with types (``adverse selection''). We will reduce to bounded-degree dominating-set, which is known to be an APX-hard problem \cite{chlebik2008approximation}. Let $H$ be a graph with $V$ vertices, each with maximum degree at most $3$. We will construct from $H$ a Bayesian game $G$ with $M = V+1$ actions for the optimizer, $N = 5$ actions for the learner, and $C = 2N$ contexts. 

For each vertex $1 \leq v \leq V$, let $\nbr(v, 0) = v$, and let $\nbr(v, 1)$, $\nbr(v, 2)$, and $\nbr(v, 3)$ be the three neighbors of $v$ in $H$, arbitrarily ordered (if a vertex $v$ has fewer than $3$ neighbors, let $\nbr(v, j) = -1$ if $j > \deg(v)$). We will label the $M = V+1$ actions for the optimizer $1, 2, 3, \dots, V$ and $\emptyset$. We will label the $N = 5$ actions for the learner $0$, $1$, $2$, $3$, and $\bself$. Finally, we will label the $C = 2V$ types $1, 2, \dots, V$, and $\overline{1}, \overline{2}, \dots, \overline{V}$.

We now describe the payoffs of $G$. We first restrict our attention to contexts $c = v$ for some $v \in [V]$. Intuitively, for these contexts the game works as follows. If the learner plays $\bself$, then both the learner and optimizer receives nothing. Otherwise, if the learner plays action $j \in \{0, 1, 2, 3\}$ and the optimizer plays action $i$, the learner loses $1/(2V)$ but receives a payment of $1$ from the optimizer if $i = \nbr(v, j)$. On the other hand, the optimizer gains $1/V$ if $\nbr(v, j)$ exists, but loses $1$ to the learner if $i = \nbr(c, j)$. Formally, we have (for $v \in V$)
\begin{eqnarray*}
u_{L}(i, j, v) &=& \Ind(i = \nbr(v, j)) - \frac{1}{2V} \\
u_{O}(i, j, v) &=& \frac{1}{V}\Ind(\nbr(v, j) \neq -1) - \Ind(i = \nbr(v, j))\\
u_{L}(i, \bself, v) &=& u_{O}(i, \bself, v) = 0 \\
\end{eqnarray*}

Now let us consider a context $c = \overline{v}$ for some $v \in [V]$. In this case, if the learner plays any action $j \in \{0, 1, 2, 3\}$, the learner receives $-1/(2V)$ and the optimizer receives $0$. On the other hand, if the learner plays $\bself$ then they receive a payment of $1$ from the optimizer if $i = v$, and the optimizer receives a payoff of $1/V$ but must pay the learner $1$ if $i=v$. Formally:

\begin{eqnarray*}
u_{L}(i, \bself, \overline{v}) &=& \Ind(i = v)  \\
u_{O}(i, \bself, \overline{v}) &=& \frac{1}{V} - \Ind(i = v)\\
u_{L}(i, j, \overline{v}) &=& -\frac{1}{2V} \\
u_{O}(i, j, \overline{v}) &=& 0 \\
\end{eqnarray*}

\noindent
Finally, we assume the distribution $\cD$ is uniform over contexts, i.e. $p_{c} = 1/2V$ for all contexts $c$.

Let us analyze $G$ as a Bayesian Stackelberg game. First note that regardless of what mixed strategy $\alpha$ the optimizer plays, when the context is $\overline{v}$ the learner's best response is to play $\bself$: this guarantees them a non-negative utility, whereas any other action they play gives them a utility of $-1/2V$. These $\overline{v}$ types cost the optimizer whenever the optimizer allocates weight to an action $i \in [V]$ instead of to $\emptyset$. In particular, the optimizer loses a total of $(1 - \alpha_{\emptyset})$ utility from these types combined.

On the other hand, assume the type $c = v$ for some $v \in [V]$. Then if there exists an action $j$ such that the optimizer places weight at least $1/(2V)$ on $\nbr(v, j)$ (i.e., $\alpha_{\nbr(v, j)} \geq 1/(2V)$), some such action will be the learner's best response: this action guarantees the learner a positive payoff, whereas the payoff for all other actions is at most zero. In this case the optimizer can get a payoff of up to $1/V - 1/(2V) = 1/(2V)$. On the other hand, if no such action exists, the learner should just play $\bself$; this gives the learner zero utility, whereas all other actions will earn the learn negative utility. This case leads to the optimizer earning $0$ utility.

From these observations, the first thing we can notice is that for each $i \in [V]$, the optimizer should allocate either weight $0$ or $1/(2V)$ to action $i$. In particular, the optimizer never loses utility by decreasing the weight of an action from above $1/(2V)$ down to $1/(2V)$, or from less than $1/(2V)$ all the way down to $0$. From now on, let's assume that $\alpha_{i} \in \{0, 1/(2V)\}$ for all $i \in [V]$.

Let $S$ be the subset of $[V]$ containing the vertices $i$ where $\alpha_{i} = 1/(2V)$. Let $\nbr(S)$ equal the set of vertices $i' \in [V]$ such that $\nbr(i', j) \in S$ for some $j$; in other words, $\nbr(S)$ is the set of vertices dominated by $S$. Then we claim that the value of the optimizer when playing this mixed strategy is equal to:

$$\frac{|\nbr(S)| - |S|}{4V^2}.$$

This is since: i) the optimizer gains utility $1/(2V)$ from each type $i \in \nbr(S)$, ii) the optimizer loses utility $1/(2V)$ from each type $\overline{i}$ with $i \in S$, and iii) each type has an equal probability of $1/(2V)$ of occurring.

Now, we claim that $\max_{S \subseteq [V]} |\nbr(S)| - |S| = V - D$, where $D$ is the size of the minimal dominating set. To see this, note that while $\nbr(S) \neq [V]$, one can weakly monotonically increase the value of $|\nbr(S)| - |S|$ by adding an element of $[V] \setminus \nbr(S)$ to $S$ (this increases $|S|$ by $1$ and $|\nbr(S)|$ by at least $1$. On the other hand, if $\nbr(S) = [V]$, then $S$ is a dominating set so $\min |S| = D$ by definition. It follows that

$$\Val(G) = \frac{V - D}{4V^2}.$$

By the results of \cite{chlebik2008approximation} (Theorem 6), there exist graphs with degree at most $3$ where it is NP-hard to distinguish decide between the cases $D \geq 0.2879V$ and $D \leq 0.2872V$. For the games generated from these graphs, it is likewise hard to distinguish whether $\Val(G) \leq 0.1781/V$ or whether $\Val(G) \geq 0.1782/V$. It is therefore NP-hard to compute $\Val(G)$ (or the associated Stackelberg strategy) to within a factor of $(1-\eps)$ for $\eps = 1/1782$.

\end{proof}

\end{document}